\documentclass[hypertexnames=false]{amsart}

\usepackage{soul}

\usepackage{tikz}
\usepackage{pgfplots}
\usepackage{amsmath,amsthm,amssymb} 
\usepackage{ae,aecompl}
\usepackage{times}
\usepackage{microtype}
\usepackage{dsfont}
\usepackage{amsfonts}

\usepackage{amscd}
\usepackage{txfonts}
\usepackage[margin=1in]{geometry} 
\usepackage[onehalfspacing]{setspace}
\usepackage{array}
\usepackage{enumerate}
\usepackage{kpfonts}
\usepackage[english]{babel}
\usepackage{comment}
\usepackage{tikz-cd}
\usepackage[numbers]{natbib}

\numberwithin{equation}{section}

\usepackage{microtype}
\usepackage{graphicx}
\usepackage{subfigure}
\usepackage{booktabs} % for professional tables
\usepackage{comment}
\usepackage{hyperref}

\newcommand{\N}{{\mathbb{N}}}
\newcommand{\R}{{\mathbb{R}}}

\newcommand{\Fc}{{\mathcal{G}}}

\newcommand{\Zc}{{\mathcal{F}}}

\newcommand{\bbR}{\mathbb{R}}

\theoremstyle{plain}
\newtheorem{theorem}{Theorem}[section]
\newtheorem{proposition}[theorem]{Proposition}

\theoremstyle{definition}
\newtheorem{definition}[theorem]{Definition}

\theoremstyle{remark}
\newtheorem{remark}[theorem]{Remark}

\begin{document}
\graphicspath{{Images/}}

\title{A Z\lowercase{e}NN architecture to avoid the Gaussian trap}




\author[L.Carvalho]{Lu\'{\i}s Carvalho}
\address{Lu\'{\i}s Carvalho, Department of Mathematics, ISCTE - Lisbon University Institute\\    Av. das For\c{c}as Armadas\\     1649-026, Lisbon\\   Portugal, and Center for Mathematical Analysis, Geometry,
and Dynamical Systems\\
Instituto Superior T\'ecnico, Universidade de Lisboa\\  Av. Rovisco Pais, 1049-001 Lisboa,  Portugal}
\email{luis.carvalho@iscte-iul.pt}

%\author{Luís Carvalho, João L. Costa, José Mourão, Gonçalo Oliveira}

\author[J.L.Costa]{João L. Costa}

\address{João L. Costa, Department of Mathematics, ISCTE - Lisbon University Institute\\    Av. das For\c{c}as Armadas\\     1649-026, Lisbon\\   Portugal, and Center for Mathematical Analysis, Geometry,
and Dynamical Systems\\
Instituto Superior T\'ecnico, Universidade de Lisboa\\  Av. Rovisco Pais, 1049-001 Lisboa,  Portugal}
\email{jlca@iscte-iul.pt}

\author[J.Mourão]{José Mourão}
\address{José Mourão, Department of Mathematics and Center for Mathematical Analysis, Geometry and Dynamical Systems, Instituto Superior T\'ecnico, Lisbon,\\ Av. Rovisco Pais, 1049-001 Lisboa,  Portugal}
\email{jmourao@tecnico.ulisboa.pt}

\author[G.Oliveira]{Gonçalo Oliveira}
\address{Gonçalo Oliveira, Department of Mathematics and Center for Mathematical Analysis, Geometry and Dynamical Systems, Instituto Superior T\'ecnico, Lisbon,\\ Av. Rovisco Pais, 1049-001 Lisboa,  Portugal}
\email{galato97@gmail.com}


\vskip 0.3in



%\printAffiliationsAndNotice{EqualContribution} % otherwise use the standard text.

\begin{abstract}

We propose a new simple architecture, Zeta Neural Networks (ZeNNs), in order to overcome several shortcomings of standard multi-layer perceptrons (MLPs). Namely, in the large width limit, MLPs are non-parametric, they do not have a well-defined pointwise limit, they lose non-Gaussian attributes and become unable to perform feature learning; moreover, finite width MLPs perform poorly in learning high frequencies.  
The new ZeNN architecture is inspired by three simple principles from harmonic analysis: 
i) Enumerate the perceptons and introduce a non-learnable weight to enforce convergence; 
ii) Introduce a scaling (or frequency) factor; 
iii) Choose activation functions that lead to near orthogonal systems. 
We will show that these ideas allow us to fix the referred shortcomings of MLPs. In fact, in the infinite width limit, ZeNNs converge pointwise, they exhibit a rich asymptotic structure beyond Gaussianity, and perform feature learning. Moreover, when appropriate activation functions are chosen, (finite width) ZeNNs excel at learning high-frequency features of functions with low dimensional domains. 
\end{abstract}

\maketitle
	
\tableofcontents

\section{Introduction}

The standard multi-layer perceptrons (MLPs)
are a fundamental building block of classical and modern Neural Networks that, unfortunately, display several undesirable features, namely:

MLPs exhibit a restrictive and somewhat pathological asymptotic behavior, a fact that is in tension with the growing use of (very) large models. More precisely, in the infinite width limit, MLPs do not have a well defined pointwise limit and are forced by the Central Limit Theorem (CLT) into a {\em Gaussian trap}; consequently large width MLPs are in essence non-parametric. Moreover, their training falls into a (fixed) kernel regime; this has consequences at the level of performance and disables feature learning which, in particular, makes it impossible to perform (standard) pre-training.

It is also quite difficult to train MLPs to learn high-frequency features making them unsuited for several machine learning tasks.

The goal of this paper is to present a simple modification of the MLP architecture, Zeta Neural Networks (ZeNNs) that overcomes these shortcomings. In fact, we shall show that, in the infinite width limit, ZeNNs remain bona fide neural networks (with fully trainable parameters) defined via a convergent pointwise limit; they avoid the Gaussian trap, displaying a rich asymptotic stochastic structure; and perform feature learning. Moreover, our experiments show that, for appropriate choices of activation functions, finite width ZeNNs perform exceedingly well in tasks involving the detection of high-frequency features of functions with low dimensional domains.

\subsection{The asymptotic behavior of MLPs}

The asymptotic behavior (large width limit) of networks is, at first hand, a conceptual issue, but the growing use and success of very large models is increasingly driving the need to better understand the way this limit 
works \cite{Jacot, Arora, Arora2, Roberts, li2023statistical, lai2023generalization, banerjee23a, carvalho2023wide, Grohs}. 
Although some successes have been achieved, many open questions remain unanswered and some paradoxical features remain unsettling. To understand them
it is instructive to start by considering the shallow network case with one-dimensional inputs and outputs (the generalization to deep networks and higher dimensional inputs and outputs will be discussed later in Section~\ref{sec:deep}). In such case the MLP architecture, with (continuous) activation $\sigma:\R\rightarrow\R$, can be encoded as a map from the space $\R^{4N}$ of trainable parameters $\theta$, the collection of weights $W^{(i)}_j$ and biases $b^{(i)}_j$, to the space of continuous functions 
\begin{equation}
\label{e1.1}
\begin{array}{rcl}
\Fc^N \, : \quad   \R^{4N} &\longrightarrow &  C(\R, \R)       \\
   \theta & \longmapsto & g_{\theta}^N  \,  ,
\end{array}
\end{equation}
defined by
\begin{align}
\label{e1.2}
\nonumber
\Fc^N(\theta)(x) &= g_{\theta}^N(x) \\ 
%= \frac 1{N^\beta}  \, \sum_{j=1}^N \, p_{\hat \theta_j}(x)  
&=  \frac 1{N^\beta} \sum_{j=1}^{N} \left( W^{(2)}_{j} \ \sigma \left(  W^{(1)}_{j} x +b^{(1)}_j \right) + b^{(2)}_j \right) \, ,
\end{align}
with $N$ the number of perceptrons/neurons and $\beta> 0$ a normalization parameter. Let us also assume that the trainable parameters are randomly initialized using iid distributions with zero mean and finite standard deviation; note that by setting $\beta=1/2$ and fixing normal distributions this is in essence equivalent to the most standard initializations. 
We also stress the fact that in this MLP architecture, all neurons are equal.
Let us now discuss some undesirable features of MLPs: 

\emph{MLPs do not converge pointwise, in the infinite width limit.} Even though the universal approximation theorems~\cite{Leshno, Hornik} guarantee that, for non-polynomial activation functions, the image of the MLP map is dense as 
$N \to \infty$, the parameters become increasingly meaningless as we approach this limit and, in effect, we are removed from the framework of parametric models. 
%practice shows that it is increasingly difficult to train the weights as $N$ grows to infinity. 
In fact, for the standard normalization $\beta=1/2$, it is well known that, with probability one at weight initialization, $\limsup_{N\rightarrow\infty}g_{\theta}^N(x)=+\infty$; consequently, the limit of these neural networks is not a neural network.
%
%\begin{itemize}
%    \item \emph{MLPs do not converge pointwise in the infinite width limit.} %$f_{\theta}^N$ \emph{does not converge pointwise} as $N \to + \infty$.
%\end{itemize} 

\emph{Wide MLPs fall into a Gaussian trap.} Nonetheless, this ($\beta=1/2$) is the famous normalization associated to the CLT, which can be applied to show that $g_{\theta}^N$ converges, in distribution, to a Gaussian random field, a fact that was considered ``disappointing" from the time of its discovery~\cite{Neal} [pgs 48 and 161].
%However, this weaker notion of convergence comes at considerable cost: 
%
%\begin{itemize}
%    \item \emph{MLPs models fall into a Gaussian trap.}%The model falls into a \emph{Gaussian trap}.
%\end{itemize} 
%
%Indeed, most information regarding architecture and initialization is lost as all non-Gaussian features of the initialization distribution are lost and so is the possibility of having a limiting parametric model; int this case, the limit of neural networks is not a neural networks. In addition to these shortcomings, there are further related problems of the MLP architecture for very large $N$. Let $\theta_*$ be a minimum of the loss. From convergence, the associated sequences $((W_{j}^{(2)})_*, (b_{j}^{(2)})_*)$ will necessarily converge to zero as $j \to \infty$, thus making such minimizing $\theta_*$ at infinite distance from any random symmetric initialization of $\theta$~\textcolor{blue}{jlca: Isto não me parece correto. Os pesos continuam simetricos vão é tornando-se cada vez mais pequenos e no limite não têm significado}. To complicate matters, the gradients of the loss function in parameter space converge to zero at the rate of $1/N$~\textcolor{blue}{E o das ZeNNs}. This is related to the so-called lazy training phenomenon \color{blue}   Inserir umas referências aqui  \color{black}.
%
Indeed, this (weaker) notion of convergence is a trap, since increasing the capacity of MLPs in terms of width forces us to work with Gaussian models, where most information regarding architecture and initialization is lost and so is the possibility of having limiting parametric models. 
%in this case, the limit of neural networks is not a neural networks. 
To complicate matters, the gradient components of the loss function in parameter space converge uniformly to zero at the rate of $N^{-1}$.
%~\textcolor{blue}{Este limite tem as suas subtilizas? Qual é a norma? E o das ZeNNs}. This is related to the so-called lazy training phenomenon \color{blue}   Inserir umas referências aqui  \color{black}.

One could try to overcome these difficulties by changing the normalization parameter $\beta$. For instance, if we set $\beta=1$, now the Law of Large Numbers applies and shows that, with probability one at initialization, the model does have a well defined pointwise limit: unfortunately this limit is  the zero function, its average, which severely weakens the model's initial expressivity. 
%So, in this case, we do have a well defined pointwise limit but one which is even less of a parametric model. 
Also, the gradient components of the loss function will decay even faster to zero as $N^{-2}$.  
In summary, changing $\beta$ may fix some problems at the expense of making others worst. 

\emph{Wide MLPs lose the ability to perform feature learning.} A celebrated asymptotic feature of MLPs is that, in the infinite width limit, they fall into a (fixed) kernel regime; the corresponding kernel is known as Neural Tangent Kernel (NTK) and is, with high probability at initialization, constant in parameter space. This clearly simplifies the analysis of the learning dynamics, making it essentially linear, but comes at a high cost; for instance, a fixed kernel disables {\em feature learning}. Basically, in the kernel learning regime of wide MLPs, the features (which, to make things more concrete, can be seen as the outputs of the last hidden layer) are randomly chosen at initialization from a fixed set of possibilities, established a priori without reference to data; the role of training is then restricted to the use of data to find appropriate coefficients to approximate the objective function in terms of the fixed features. In particular, we cannot perform {\em transfer learning}, where a set of features learned from a given data set is later reused and fine-tuned  
in a different learning task. Feature learning can thus be related to the existence of a generalized kernel that depends on parameters and therefore can be adapted, via training; it is this later perspective that we will adopt here.   

%\begin{itemize}
%    \item \emph{MLPs lose the ability to perform feature learning, in the infinite width limit.}
%\end{itemize} 

There is one further issue with the MLP architecture that, although of a different flavor, we would like to address:  
%This may be a feature of the previously mentioned difficulty in learning, but there is growing experimental evidence \color{blue}  Inserir umas referências aqui  \color{black} that:
%\begin{itemize}
%    \item MLP have \emph{difficulties in learning high frequencies.} %the MLP architecture has \emph{difficulty learning hight frequencies}.
%\end{itemize}

\emph{MLPs have difficulties in learning high frequencies.}
Is is a well known fact that, both narrow and wide, MLPs suffer from a ``spectral bias"~\cite{Rahaman:spectral_bias}, i.e., a tendency to learn the low frequency structure of data first, in the earlier stages of training, and then struggle to detect higher frequencies. This is in line with the standard intuitive picture that MLPs have the tendency to prioritize the ``low complexity" structures that recur across an entire data set. In practice, however, there are plenty of examples, from time series analysis to image processing (see Section~\ref{sec: Experiments}) where the high frequency structure is of paramount importance, making MLPs, alone, unsuited for such tasks. 
%As we will see our ZeNN architecture provides novel and promising ways to solve this problem, at least in the context of low dimensional inputs. Note that other successful strategies have been put forward to handle this issue, we refer to the ``related results" discussion below for further information.         

\subsection{The introduction of ZeNNs}

Hopefully, it is by now clear that to avoid the referred difficulties we need to change the MLP architecture~\eqref{e1.2}. 
Inspired by harmonic analysis we propose the following new architecture:
%\st{A functional analytic view reveals that, in the limit $N \to \infty$, we should be approximating functions with pointwise convergent functional series. A necessary condition for such a series, in $j \in \mathbb{N}$, to be convergent is that its coefficients $W^{(2)}_j$ and $b^{(2)}_j$ converge to zero (as $j\rightarrow\infty$), thus breaking in a drastic way the permutation symmetry of the map} %(\ref{e1.2})}
%~\textcolor{blue}{(Referir resultados de neurociências que mostram que os neurónios não são todos iguais e resultado de modelos artificias assimétricos.)}.  
%a change in the above framework that fixes the previously mentioned problems: 
%\st{In addition, with our proposal the parameters are likely to start much closer to the minimum, and the gradients do not decrease with the increase on $N$ thus improving their training.}

\begin{definition}
\label{ZeNN_def}
Zeta Neural Networks (ZeNNs) with one hidden layer (scalar inputs and scalar outputs) are given by
\begin{align}
\label{e1.5}
\nonumber
\Zc^N(\theta)(x) &= f_{\theta}^N(x) \\
%=  \, \sum_{j=0}^N \, p^{(j)}_{\theta_j}(x)  
&= \, \sum_{j=0}^{N}  \frac{1}{j^\alpha}   \, \left( W^{(2)}_{j} \ \sigma\left( W^{(1)}_{j}  j x + b^{(1)}_j \right) + b^{(2)}_j \right) \, ,
\end{align}
where $\alpha > 0$ is a fixed convergence parameter. 
\end{definition}

\begin{remark}
\begin{itemize}
\item[]
\item Note the explicit asymmetry with respect to permutation of the neurons. So, instead of averaging over identical perceptrons, a ZeNN sums over different frequency (or rescaled) perceptrons 
$$p^{(j)}_{\hat \theta}=  W^{(2)} \ \sigma \left( W^{(1)}_{j}  j x + b^{(1)}_j \right)  + b^{(2)}   \, ,$$ 
which are weighted by a non-learnable decaying factor $j^{-\alpha}$, introduced to enforce (pointwise) convergence. 
\item Note also that we can easily apply standard results~\cite{Leshno} to show that, provided $\sigma$ is non-polynomial, ZeNNs are also Universal approximators, i.e., the image of $\Zc^N$ is dense, as $N\rightarrow \infty$.  
\item For the sake of generality it is nice to impose almost no restrictions at the level of the activation functions a priori. Later tough, we will see that, in some situations, it is relevant to choose the activations in such way that the collection of maps 
%$x\mapsto\sigma \left( W^{(1)}_{j}  j x + b^{(1)}_j \right)$
$x\mapsto\sigma \left(j x \right)$
forms an $L^2$--orthogonal system.  
\end{itemize}
\end{remark}

\subsection{ZeNNs' contributions}

This article proves that, under appropriate yet quite general assumptions, ZeNNs do not suffer from the previously identified asymptotic pathologies of MLPs. More concretely, in Section \ref{sec:main_results} we shall see that, they \emph{converge pointwise} leading to well defined limiting parametric models (see Theorem~\ref{thm:convergence}), that they exhibit a wide variety of stochastic behaviors which, in particular, allows them to \emph{escape the Gaussian trap} set by the CLT (see Theorem~\ref{thm:Cumulants} and Remark \ref{rem:Cumulants at x=0}), and that ZeNNs perform feature learning even in their infinite width limit; the later is established via the introduction of a natural adaptation of the NTK to the current context, the ZeNTK. Section \ref{sec:deep} explores proposals for deep ZeNNs, and \ref{sec: Experiments} makes experimental tests, using in particular the referred proposals of deep ZeNNs, which indicate that, as expected, ZeNNs \emph{can learn high frequencies}, for problems with low dimensional input space.

%\color{black}

\subsection{Related results}

The study of wide neural networks has a long tradition, with by now classical results establishing that all functions, in appropriate regularity classes, can be arbitrarily approximated by sufficiently wide MLPs with non-polynomial activations~\cite{Leshno, Hornik}, and that MLPs converge, in distribution, to Gaussian processes, when the width of all hidden layers goes to infinity~\cite{Neal, Lee}. 
Still concerning this limit, a fundamental finding was done in~\cite{Jacot}, where it was shown that the evolution of outputs during training is determined by a kernel, the NTK, which, with high probability, is independent of parameters. This result was later extended in~\cite{Yang2019ScalingLO,Arora}. 

The difference in performance between finite and infinite width networks has been the subject of considerable attention with somewhat conflicting observations and results~\cite{Lee:finite_vs,Yao:vs,Aitchison,Pleiss}; nonetheless, the most common trend being that finite-width networks tend to outperform their infinite width counterparts. Several reasons have been advocated to explain these differences: In~\cite{Pleiss} it was for instant argued, using Deep Gaussian  Networks, that in the infinite width limit depth becomes meaningless. Other results~\cite{Yang:tensor_iv,Aitchison} have observed that, in the infinite width kernel regime, feature learning is disabled; these works also argue that this impairs performance and has other negative consequence such as removing the possibility of using pre-training as a viable strategy. 
%For a somewhat opposing argument\textcolor{blue}{VER - Infinite Width Models That Work: Why Feature
%Learning Doesn’t Matter as Much as You Think.} 

Before ours, proposals had already been presented to avoid some of the difficulties related with the infinite width asymptotic behavior of MLPs. ~\cite{Yang:tensor_iv} keeps the MLP architecture but changes the normalization in a clever way; the proposed $abc$--normalization acts both layer-wise, but also at the level of the learning rate and by fine-tuning the $abc$ parameters this allows to restore feature learning in the infinite width limit; nonetheless, a cost to pay is that models are always set to zero at initialization. The use of bottleneck layers~\cite{Aitchison}, layers that are kept at finite width even when others are taken to their infinite width limit, was also advocated as a way to retain feature learning.

Deviations from Gaussianity that occur in the large but finite width case have been studied for instance in~\cite{Roberts}, using techniques from effective quantum field theory, and in~\cite{carvalho2023wide} using the Edgeworth expansion.   

The spectral bias of MLPs was identified in~\cite{Rahaman:spectral_bias} and later related to the spectra of the NTK in~\cite{cao:spectral_bias,Tancik:fourier_features,Wang:fourier_features}. 
%\textcolor{blue}{VER Theory of the Frequency Principle for General Deep Neural Networks}. 
Several strategies were proposed to overcome this obstacle. Specially relevant to us here is the introduction of a non-learnable Fourier features embedding, which provides an elegant solution~\cite{Tancik:fourier_features} and leads to very interesting results, for problems with low dimensional input domains (see Sections~\ref{sec:deep} and~\ref{sec: Experiments} for more details). A related approach, leading to remarkable results in diverse tasks, is provided by SIREN networks~\cite{sitzmann2019siren},  whose first layer can be described as a learnable Fourier features layer, followed by hidden dense layers with sine activations. We would also like to refer to~\cite{Liu:multi_scale} where non-learnable radial scale parameters are introduced, at the level of the inputs, to aid PINNs in determining high frequencies of solutions of specific differential equations.

%%%%%%%%%%%%%%%%%%%%%

\section{Main theoretical results: MLP versus ZeNN}\label{sec: Theoretical Results}
\label{sec:main_results}

In this section we present our main theoretical results, whose proofs can be found in the appendices. The first establishes that, under suitable conditions, ZeNNs converge uniformly and therefore pointwise. We consider a sequence of parameters $\hat{\theta}_1, \ldots , \hat{\theta}_N , \ldots $ which we regard as random variables with values in $\mathbb{R}^4$. Let $\hat{\nu_j}$, for $j \in \mathbb{N}$, be the measures on $\mathbb{R}^4$ associated with these random variables. A theorem of Kolmogorov guarantees that these induce a measure on $(\mathbb{R}^{4})^{\mathbb{N}}$ which we shall denote by $\nu^\infty$.

\begin{theorem}[ZeNNs converge poitwise]\label{thm:convergence}
 Let $\sigma$ be continuous and have at most polynomial growth of order $k \in \mathbb{N}$, i.e it satisfies $|\sigma(x)| \lesssim 1+|x|^k$. Further assume the measures $\widehat{\nu_j}$ to have uniformly bounded first moments. Then, for $\alpha>k+1$, the sequence of asymmetric maps $\widehat\Zc^N$ 
    \begin{equation}
\label{e1.12}
\begin{array}{rcl}
\widehat \Zc^N \, : \quad   \left(\R^4\right)^{\N} &\longrightarrow &  C(\R, \R)       \\
   \theta & \longmapsto & \Zc^N \circ P^{(N)} \, (\theta) =\,  f_{\theta}^N  
\end{array}
    \end{equation}
     converges uniformly in compact subsets of $\mathbb{R}$, and a.e. in $\left(\left(\R^4\right)^{\N}, \nu^\infty\right)$. In particular, for $\nu^\infty$-a.e. $\theta \in \left(\R^4\right)^{\N}$ and all $x \in \mathbb{R}$
    \begin{align}
\widehat \Zc^\infty(\theta)(x)  &= \lim_{N \to \infty} \widehat \Zc^N(\theta)(x) \\
&=
\sum_{j=1}^{\infty}  \frac{1}{j^\alpha}   \, \left(  W^{(2)}_{j} \ \sigma\left( W^{(1)}_{j}  j x + b^{(1)}_j \right) +  b^{(2)}_j \right) \, .
    \end{align}
\end{theorem}

We turn now to a second result, which shows that our proposed ZeNNs avoid the CLT Gaussian trap and can retain non-Gaussian features even in the infinitely wide limit. The statement requires a little preparation. 

Consider a ZeNN, with $N \in \mathbb{N}$ hidden neurons,
$$f_{\theta}^N(x) = \sum_{j=1}^N \frac{1}{j^\alpha} p_{\hat \theta_j}(jx)   ,$$
where for $\hat \theta=(W^{(2)},b^{(2)}, W^{(1)},b^{(1)})$
$$p_{\hat \theta}(x)  = W^{(2)}\sigma(W^{(1)} x + b^{(1)}) + b^{(2)}, $$ 
denotes a standard perceptron. Now, consider all tuples of the parameters $\hat \theta_j \sim \hat \theta$ to be independent and identically distributed random variables. Then, for any fixed $x \in \mathbb{R}$, we can regard both $p_{\hat \theta}(x)$ and $f_{\theta}^N(x)$ also as random variables. 

\begin{theorem}[ZeNNs escape the Gaussian trap]\label{thm:Cumulants}
Let $x \in \mathbb{R}$ and $r \in \mathbb{N}$ be fixed. Denote the $r$-cumulant of $f_{\theta}^N(x)$ by $\lambda^{(r)}_N(x)$ and that of $p_{\hat \theta}$ by $\lambda^{(r)}(x)$, then
    $$\lambda^{(r)}_N(x) = \sum_{j=1}^N \frac{\lambda^{(r)}(jx)}{j^{r\alpha}}.$$
Furthermore, if $\sigma$ has at most polynomial growth of order $k \in \mathbb{N}$, i.e.  satisfies $|\sigma(x)| \lesssim 1+|x|^k$, and the moments of $\hat \theta$ are bounded, then, for $\alpha>k+1$ the ZeNN converges pointwise a.e. and the limiting cumulants are finite and satisfy
    $$\lim_{N \to + \infty} \lambda^{(r)}_N(x) = \sum_{j=1}^{+\infty} \frac{\lambda^{(r)}(jx)}{j^{r\alpha}}.$$
\end{theorem}

\begin{remark}\label{rem:Cumulants at x=0}
    For $x=0$ we have
    $$\lambda^{(r)}_N(0) = \sum_{j=1}^N \frac{\lambda^{(r)}(0)}{j^{r\alpha}} = \lambda^{(r)}(0) \ \sum_{j=1}^N \frac{1}{j^{r\alpha}},$$
    and under the assumption that $r\alpha>1$ we can take the limit and obtain
    $$\lim_{N \to + \infty} \lambda^{(r)}_N(0) =  \lambda^{(r)}(0) \ \sum_{j=1}^{+\infty} \frac{1}{j^{r\alpha}}.$$
    In particular, it does not vanish if $\lambda^{(r)}(0) \neq 0$. We therefore find that $f_{\theta}^N(0)$ does not converge to a Gaussian distribution if any of the cumulants of order at least three of $p_{\hat \theta}(0)$ is nonzero. More details can be found in Appendix \ref{app:Cumulants}, where we also construct characteristic functions of ZeNNs exhibiting explicit non-Gaussian behavior. 

    In fact, we further find that little information is lost when passing to the limit $N \to +\infty$. Indeed, suppose $\hat \theta$ and $\hat \vartheta$ are two sets of parameters drawn from two different distributions such that at least one of the cumulants of $p_{\hat \theta}$ and $p_{\hat \vartheta}$ are different. Then, the $N \to +\infty$ limit of two ZeNNs with $\hat \theta_i \sim \hat \theta $  and $\hat \vartheta_i \sim \hat \vartheta$ follow different distributions.
\end{remark}

\begin{remark}[Comparison with MLPs]
    Recall that for a MLP given by the formula $g_{\theta}^N(x) = \frac{1}{\sqrt{N}} \sum_{j=1}^N p_{\hat \theta_j}(x)$, as $N \to + \infty$ the CLT guarantees that, for iid perceptrons of zero average, $g_{\theta}^N(x)$ converges to a Gaussian. In fact, as shown in remark \ref{rem:Cumulants of Standard NN} of the appendix \ref{app:Cumulants} we have
     $$\lambda^{(r)}_N = \frac{\lambda^{(r)}(x)}{N^{r/2-1}}, $$
     which converges to zero for $r>2$. This is what we have called the CLT Gaussian trap, or the loss of non-Gaussian features. Indeed, apart from the first two moments, all the information of the parameter distribution gets lost in the infinitely wide limit.
\end{remark}

Finally, we state our third main theoretical result. For ZeNNs, the relevant kernel responsible for training in function space is a modification of the NTK for MLPs~\cite{Jacot} which we call ZeNTK.  
To deduce/define it, consider the loss function
$$\mathcal{L}^N(\theta) = \frac{1}{2} \sum_{\mu} \left( \Zc^N(\theta)(X_\mu) - Y_\mu   \right)^2 ,$$
with parameters $\lbrace \theta(t) \rbrace_{t \geq 0}$ trained via gradient descent, i.e.
$$\dot{\theta}(t)=-\nabla \mathcal{L}^N ( \theta(t) ) . $$
Then, we find that
\begin{align} \label{label_evolution}
    \partial_t \Zc^N(\theta(t))(x) = - \sum_{\mu} K_\theta^N (x,X_\mu) ( \Zc^N(\theta)(X_\mu) - Y_\mu),
\end{align}
where 
\begin{equation}
K_\theta^N (x,y) := \sum_{p=1}^{4N} \frac{\partial \Zc^N}{\partial{\theta_p}}(x) \frac{\partial \Zc^N}{\partial{\theta_p}}(y)
\end{equation}
is the desired kernel, 
and will be referred to as the Zeta Neural Tangent Kernel (ZeNTK).

%Here we shall denote it by $K_\theta^N$, where $N$ denotes the width and $\theta \in \mathbb{R}^{4N}$ the parameters at which it is evaluated. 
As previously discussed, feature learning is related to the existence of a non-constant, i.e. parameter dependent, (generalized) kernel. The idea being that the kernel encodes the features and, therefore, {\em feature-learning} corresponds to the ability of changing the kernel through training. Recall that in standard kernel methods, the kernel is fixed apriori and, in the context of MLPs, that is exactly what happens in the infinite width limit, where the NTK becomes constant (see also \cite{Yang2} for further discussion of this issue).      
%In context of MLPs the dynamical dichotomy of \cite{Yang2} makes this suggestion. This dichotomy states that, during training in the infinite width limit, the function associated with a neural network either: evolves linearly through a constant kernel in function space, or admits feature learning, but not both. 
%We will use this dichotomy as a definition (see also discussion in the introduction), 
We will now show that, in the context of ZeNNs, feature learning is preserved, even in the infinite width limit, by showing that the ZeNTK is not constant when $N \to +\infty$.

%Hence, it will suffice to prove that the ZeNTK is not constant in the infinite width limit $N \to +\infty$.

\begin{theorem}[ZeNNs learn features, even in the infinetly wide limit]\label{thm:Features Appendix}
Let $x, y \in \mathbb{R}$ and suppose $\dot{\sigma} \neq 0$ in a non-empty interval. Then, there is an open set $U \subset \mathbb{R}^{4\mathbb{N}}$ such that the map
$$\theta \in U \mapsto K_{\theta}^N (x,y)\;,$$
is not constant. We also have that for all $\theta \in U$ such that the limit $\lim_{N \to +\infty}K_{\theta}^N (x,y) = K_{\theta}^\infty(x,y)$ exists, $\theta \mapsto K_{\theta}^\infty(x,y) $ is also not constant.

Furthermore, if $\dot{\sigma} \neq 0$ except at a finite set of points, then the set $U$ can be taken to have full measure.
\end{theorem}

\section{Deep ZeNNs}
\label{sec:deep}

In this section we provide several proposals on how to include ZeNN like layers in the context of deep Neural Networks.  
In that setting  a layer is simply a function
$$\Phi^{(k)}:\R^{N_k} \rightarrow \R^{N_{k+1}},$$ 
and a (deep) Neural Network is the result of composing several layers 
\[
\Psi(x)=\Phi^{(L)} \circ \ldots \circ \Phi^{(1)}(x)\;.
\]
%
%A Neural Network will have ZeNN like features if it is constructed by iterating ZeNN layers with themselves or other architectures

So, to describe new ZeNN like layers we only need to show how to construct appropriate functions $\Phi:\bbR^n\rightarrow\bbR^m$. 
We will now discuss the three concrete proposals that were used as first hidden layers in our experiments (see Section~\ref{sec: Experiments}).

\subsection{oZeNNs}
For simplicity we start with a layer with two-dimensional inputs. In that case, a natural generalization of the hidden layer in Definition~\ref{ZeNN_def} is to consider a map $\Phi:\bbR^2 \rightarrow \bbR^{N\times N}$, defined by
%whose output is defined by the matrix
%
%\begin{align}
%\label{oZeNN}
%&\Phi(x_1,x_2) =\\ 
%&\left[ \frac{1}{(j_1j_2)^{\alpha}}\sigma\big(W_{1j_1}j_1x_1+W_{2j_2}j_2x_2 + b_{1j_1}+b_{2j_2} \big)\right]_{(j_1,j_2)} 
%\end{align}
%
\begin{equation}
\label{oZeNN}
\Phi(x_1,x_2) =
\left[ \frac{1}{(j_1j_2)^{\alpha}}\sigma\big(W_{1j_1}j_1x_1+W_{2j_2}j_2x_2 + b_{1j_1}+b_{2j_2} \big)\right]_{(j_1,j_2)} 
\end{equation}
%
%
%\begin{equation}
%\label{oZeNN}
%\Phi(x_1,x_2)=\left[ \frac{1}{(j_1j_2)^{\alpha}}\sigma\big(W_{1j_1}j_1x_1+W_{2j_2}j_2x_2 + b_{1j_1}+b_{2j_2} \big)\right]_{(j_1,j_2)} \;,
%\end{equation}
%
where the output is a matrix with indices $1\leq j_1,j_2\leq N$, $N$ is a hyperparameter fixing the maximal scale/frequency factor,  $\alpha\geq0$ is a decaying factor, just as before, and 
$[W_{ij}],[b_{ij}]\in \bbR^{2\times N}$ are matrices of learnable weights and biases. We refer to these layers as \emph{oZeNNs}, with the ``o'' coming form ``orthogonal'', since if we chose the activation function $\sigma$ to be, for instance, the sine function, set the weights to unity and the bias to zero, then the elements of the matrix provide a $L^2$--orthogonal system of functions. 

This has a natural generalization to higher dimensional inputs  by mapping $\vec{x}=(x_1,\ldots,x_d)\in\bbR^d$ to the $d$-tensor whose $\vec{j}=(j_1,\ldots,j_d)\in\mathbb{Z}^d_N$ entry is defined by 
\begin{equation}
\label{oZeNN_general}
\Phi(\vec{x}) = \left[ \frac{1}{(\prod_i j_i)^{\alpha}}\sigma\big(\sum_{i=1}^d(W_{ij_i}j_ix_i + b_{ij_i} )\big)\right]_{\vec{j}\in \mathbb{Z}^d_N}\;,
\end{equation}
where the $N^d$ tensor entries are indexed by the elements of $\mathbb{Z}^d_N$, the set of $d$--dimensional vectors with entries in $\mathbb{Z}_N=\{1,\ldots,N\}$. Now $[W_{ij}]$ and $[b_{ij}]$ are matrices in $\bbR^{d\times N}$. Unfortunately such suggestion is at the mercy of the curse of dimensionality; notice that this is not directly related to the growth of learnable parameters in this layer, which is ``only" of order $Nd$, but to the fact that the output of the layer is $\#\mathbb{Z}^d_N=N^d$ dimensional; therefore, if we connect it to a new dense layer with $N$ neurons, then the number of parameters in the new layer will be of order $N^{d+1}$.
In practice, even for two-dimensional inputs and a large $N$ the oZeNNs can become computationally expensive (see Section~\ref{sec: Experiments}).      

\subsection{radZeNNs}

A simple strategy to attenuate the burden of the curse of dimensionality is to choose only a subset of scales, which in our framework means to choose a subset of vectors in $\mathbb{Z}^d_N$. A first natural approach is to consider only the diagonal elements, i.e. the vectors of the form $(j,\ldots,j)$, for $j\in\mathbb{Z}_N$. In that case our layer becomes a map from  $\bbR^d$ to $\bbR^N$ with ouptut given by
\begin{equation}
\label{radZeNN}
\Phi(\vec{x}) = \left[ \frac{1}{j^{\alpha}}\sigma\left(\sum_{i=1}^d W_{ij}jx_i + b_{j} \right)\right]_{j\in \mathbb{Z}_N}\;. 
\end{equation}
This approach prioritizes the radial directions, so we name it {\em radZeNN}, and, with the exception of the convergence factor $j^{-\alpha}$, is similar to the construction proposed in~\cite{Liu:multi_scale}. Unfortunately, this corresponds to the loss of to much relevant information in terms of frequencies and these layers end up underperforming in our experiments (see Section~\ref{sec: Experiments}).    

\subsection{randoZeNNs}

To obtain better results while at the same time avoiding the curse of dimensionality, we consider yet another proposal, that we once again explain in the context of two-dimensional inputs, but can be easily generalized to higher dimensions. This proposal consists in fixing an integer $M$, such that $N\leq M\ll N^2$, and then randomly choosing a sequence of pairs of integers $(j_{1r},j_{2r})$, $1\leq r\leq M$, satisfying $1\leq j_{1r},j_{2r}\leq N$. We can then define a new layer by mapping $(x_1,x_2)\in\bbR^2$ to $\bbR^M$, with output defined by  
\begin{equation}
\label{randoZeNN}
%(x_1,x_2) \mapsto 
\Phi(\vec{x})=\left[ \frac{1}{(j_{1r}j_{2r})^{\alpha}}\sigma\big((j_{1r}x_1W_{1r}+j_{2r}x_2W_{2r}) + b_{r} \big)\right]_{r\in \mathbb{Z}_M}
\end{equation}
where now $(W_{1r}),(W_{2r})$ and $(b_r)$, $1\leq r \leq M$, are vectors, in $\bbR^M$, of learnable weights and biases. We refer to these layers as \emph{randoZeNNs}. 

If we set the activation function to sines or cosines, there is a clear connection here with random Fourier features (FF)~\cite{Tancik:fourier_features}, which uses a non-learnable embedding layer 
\begin{equation}
\label{FF}
\bbR^2\ni x\mapsto [\sin(2\pi Bx),\cos(2\pi Bx)]\;,
\end{equation}
with the set of frequencies being randomly chosen by sampling $B\sim {\mathcal N}(0,\rho^2)$. 

The relevant differences between our randoZeNNs and FF are: the inclusion of a convergence factor in our randoZeNN layer and the fact that we are hard-coding the way our scale/frequency directions in frequency space, the $(j_{1r},j_{1r})\in\mathbb{Z}^2$, enter the network.  Concerning this last fact we note that, in practice, both proposals provide different ways to generate relevant frequency features. As we will see, at least in a simple image regression task, our randoZeNNs perform slightly better than Fourier features, and, interestingly, both seem to perform (slightly) better than the full oZeNN layer.

\subsection{KAZeNNs} 

We will now consider another (untested) proposal to achieve deepness with ZeNN like features by using the Kolmogorov--Arnold Theorem.  

Recall that this theorem guarantees that an arbitrary function of $d$ variables can be represented using $(2d+1)(d+1)$
functions of one variable:
\begin{equation}
%\label{e9.1}
f(x_1, \dots, x_d)
= \sum_{q=1}^{2d+1} \, \Phi_q\left(\sum_{p=1}^{d} \, \phi_{q,p}(x_p) \right)\;.
\end{equation}

Inspired by this result  \cite{liu2024kan} proposes an adaptation of the MLP architecture that they called Kolmogorov--Arnold Networks (KAN). 
Following this paper we will use the representation
of functions of $d$ variables 
by functions of one variable 
using maps $\tilde \Phi: \mathbb{R}^d  \longrightarrow  \mathbb{R}^m$ of the form
\begin{equation}
\label{eqKAN}
(x_1, \dots, x_d)  \longmapsto  \left(\sum_{p=1}^d \ \phi_{1p}(x_p), \dots , \sum_{p=1}^d \ \phi_{mp} (x_p)\right). 
\end{equation}
which can be more compactly represented in matrix notation using 
$$
\tilde \Phi=\Big[\;\phi_{q,p}\;\Big]_{\substack{p=1,\ldots,d \\ q=1,\ldots,m}}\;.
$$
 A KAN  can then be obtained by composing (at least) two of these maps
\begin{equation}
\label{KAN}
\Psi(x)=\tilde \Phi^{(L)} \circ \ldots \circ \tilde \Phi^{(1)}(x)\;,
\end{equation}
where $\tilde \Phi^{(\ell)}$ are $n_\ell \times n_{\ell-1}$ matrices (of functions). We define KAZeNNs as KANs~\eqref{KAN}  with the functions 
$\phi_{q,p}$ chosen to be shallow scalar ZeNNs,
as in \eqref{e1.5}.  

If we use (shallow) ZeNNs with $N$ units, and set all $n_k\sim d$, then the corresponding KAZeNN has $O(LN d^2)$ parameters. 
So, in conclusion, KAZeNN are a promising candidate to fight the curse of dimensionality.

%%%%%%%%%%%%%%%%%%%%%%%%%%%%%%%%%%%%%%%%
\begin{figure}
    \centering
    \subfigure[]{\includegraphics[width=0.35\textwidth]{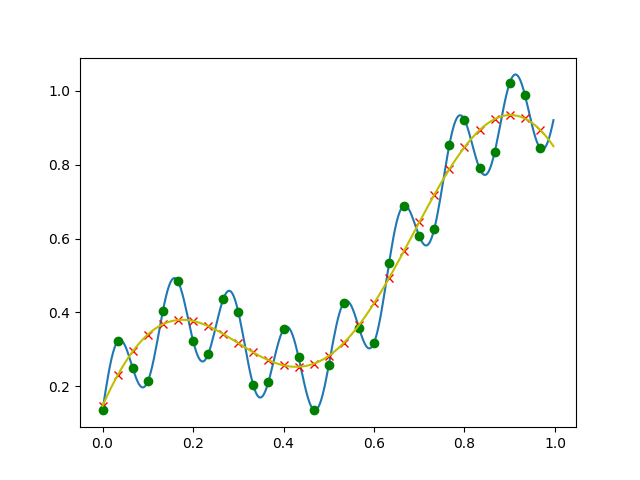}} 
    %\subfigure[]{\includegraphics[width=0.24\textwidth]{plot_standardNN_1024_na_sig_sinteticData2_1000000ep_lr0p25_reg0.png}} 
    \subfigure[]{\includegraphics[width=0.35\textwidth]{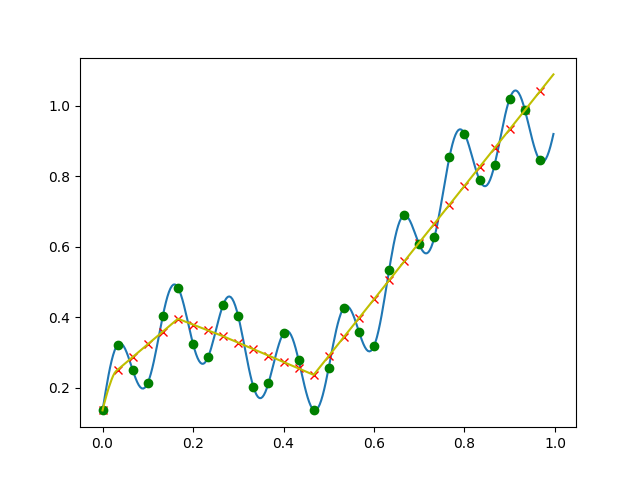}}
    \subfigure[]{\includegraphics[width=0.35\textwidth]{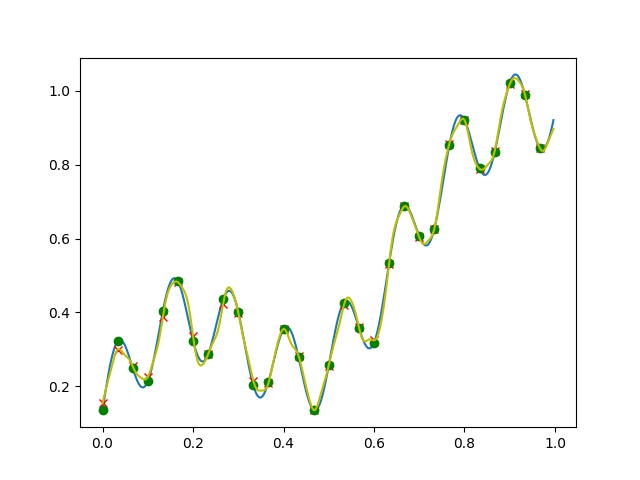}}
    \subfigure[]{\includegraphics[width=0.35\textwidth]{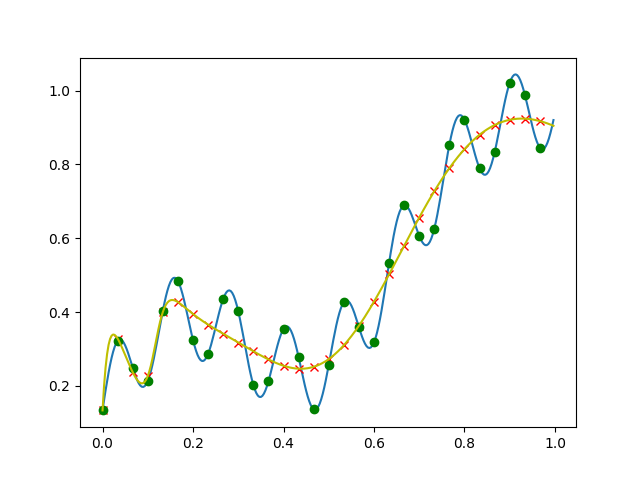}}
    \caption{Learning synthetic data generated by~\eqref{sintData2}, using different MLPs and ZeNNs. The blue line is the graph of the generating function, the green balls correspond to the training data, the red crosses are the predictions of the trained model and the yellow line is the graph of the model function. 
    (a) MLP with sine activation, 128 perceptrons, trained for 50.000 epochs. 
    %(b) MLP with sigmoid activation, 1024 perceptrons, trained for 1.000.000 epochs.
    (b) MLP with ReLU activation, 1024 perceptrons, trained for 1.000.000 epochs.
    (c) ZeNN network with $\alpha=1.1$, sine activation, with 64 perceptrons, trained for 15.000 epochs.
    (d) ZeNN network with $\alpha=1.1$, sigmoid activation, with 1024 perceptrons, trained for 1.000.000 epochs.
    }
    \label{fig:synt_data}
\end{figure}
%%%%%%%%%%%%%%%%%%%%%%%%%%%%%%%%%%%%%%%%%

\subsection{Deep issues}

In the context of deep networks, ZeNNs are most naturally suited for the first hidden layer, where they can efficiently encode frequency-based features of the input. The ZeNN structure allows to extract meaningful representations, and the deeper layers learn  other features rather than directly handling raw frequency components.

% In the context of deep networks, ZeNN like layers (with a scale factor) are more naturally used as first hidden layers, as they include as learnable 
% parameters all possible frequencies of the
% inputs.

In fact, in our experiments below we will consider networks with 4 hidden layers where the first hidden layer is a ZeNN like layer (oZeNN, radZeNN or randoZeNN), and the remaining 3 layers are MLPs. Note that, in these experiments, we are using ``narrow'' layers, and for that reason the asymptotic problems associated to MLPs, discussed in previous sections, are not an issue. If one moves into the context of wide layers then we have to deal with the asymptotic pathologies of MLPs. In such case we propose to use a ZeNN like layer with no scale factor, or equivalently an MLP with a ZeNN convergent factor. This corresponds to a new layer defined by a map $\Phi:\bbR^d\rightarrow \bbR^N$ whose output is given by          
\begin{equation}
\label{convMLP}
\Phi(\vec{x})=\left[ \frac{1}{j^{\alpha}}\sigma\left(\sum_{i=1}^d W_{ij}x_i + b_{j} \right)\right]_{j\in \mathbb{Z}_N}\;. 
\end{equation}

In the shallow scalar case, this gives rise to the network
\begin{equation}
\label{shallow_convMLP}
\nonumber
 \tilde f_{\theta}^N(x)  
= \, \sum_{j=0}^{N}  \frac{1}{j^\alpha}   \, \left( W^{(2)}_{j} \ \sigma\left( W^{(1)}_{j}  x + b^{(1)}_j \right) + b^{(2)}_j \right) \,,
\end{equation}
%
%
%\begin{align}
%\label{shallow_convMLP}
%\nonumber
%\tilde \Zc^N(\theta)(x) &= \tilde f_{\theta}^N(x) \\
%=  \, \sum_{j=0}^N \, p^{(j)}_{\theta_j}(x)  
%&= \, \sum_{j=0}^{N}  \frac{1}{j^\alpha}   \, \left( W^{(2)}_{j} \ \sigma\left( W^{(1)}_{j}  x + b^{(1)}_j \right) + b^{(2)}_j \right) \, .
%\end{align}
%
that corresponds to~\eqref{e1.5} without the scale factor. It can be shown that these networks satisfy properties analogous to those presented in Section~\ref{sec: Theoretical Results}.

\section{ZeNNs learn high frequencies - experimental results}
\label{sec: Experiments}

In this section we report on some experimental results that show that ZeNNs excel at learning high frequency features. 

\subsection{1d problems with shallow networks}
\label{1d-Exps}

First we consider one-dimensional regression problems and compare the performance of various shallow networks, i.e., networks with only one hidden layer. We start by considering synthetic data generated by the function 
\begin{equation}
\label{sintData2}
y = x + 0.125\sin(10x) + 0.135\cos(5 x) + 0.115\sin(50 x)\;.
\end{equation}
As expected MLPs struggle to fit the data due to the presence of ``high'' frequency features. This happens even if we increase the number of perceptrons considerably, or increase the number of training epochs. Our results also suggest that changing the activation function has no significant impact on the performance of (shallow) MLPs and in the end these models are only able to learn the low frequency structure of the underlying function; we stress that this happens even if we use the sine function as activation (see Figure~\ref{fig:synt_data}, subplots a) and b)). 

On the contrary ZeNN models with sine activations excel at solving this problem (see Figure~\ref{fig:synt_data}, subplot c)), fitting the objective curve extremely well while using few resources, both in terms of parameters and learning time. We note however that with ZeNNs the choice of activation function plays a crucial role.  In fact, if we use ReLU or sigmoid activations the performance of the model decreases dramatically (see Figure~\ref{fig:synt_data}, subplots d)). 
%We associate this to the fact that when we choose the activation to be a sine function then $x\mapsto \sigma(jx)$ is an $L^2$-orthogonal system. 

To test our ZeNNs against one--dimensional real world data we consider the standard task of fitting the temperature timeseries of the Jena climate Dataset~\cite{Jena_data}.
%the weather timeseries dataset recorded at the weather station at the Max Planck Institute for Biogeochemistry in Jena, Germany - Adam Erickson and Olaf Kolle, www.bgc-jena.mpg.de/wetter. 
As we can see in Figure~\ref{jena1} a simple shallow ZeNN can easily detect the high-frequency structure of the data, providing a good fit of the timeserires.

\subsection{Image regression}
\label{subsec:image}

%%%%%%%%%%%%%%%%%%%%%%%%%%%%%%%%%%%%%%%%%%%%%%%%%%%%%
\begin{figure}[t]
\centering
\includegraphics[width=8cm]{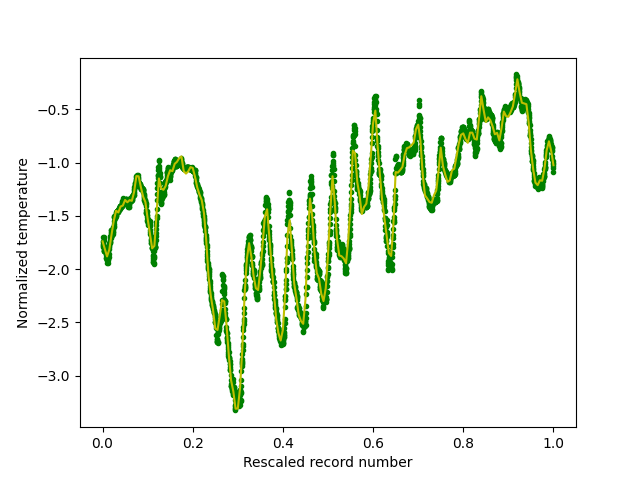}
\caption{ 
Curve fitting of the first 3000 temperature Jena data points, represented by green balls. Our fit (yellow line) uses a shallow ZeNN network with sine activation, 512 perceptrons and $\alpha=1.1$, trained for 250.000 epochs. 
}
\label{jena1}
\end{figure}
%%%%%%%%%%%%%%%%%%%%%%%%%%%%%%%%%%%%%%%%%%%%%%%%%%%%%%

%%%%%%%%%%%%%%%%%%%%%%%%%%%%%%%%%%%%%%%%%%%%%%%
%
\begin{figure}
    \centering
    \subfigure[]{\includegraphics[width=0.25\textwidth]{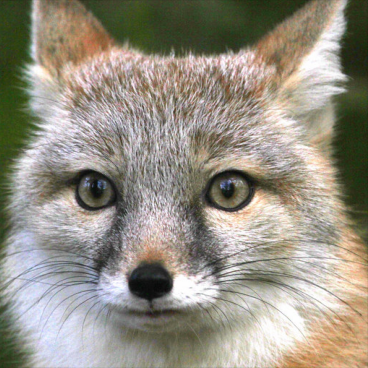}} 
    \subfigure[]{\includegraphics[width=0.25\textwidth]{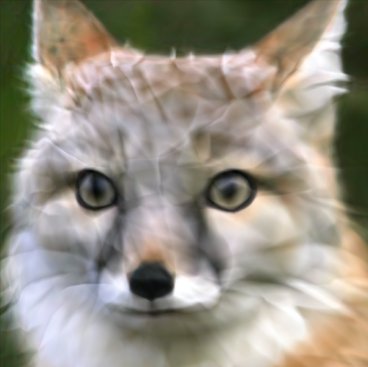}}
    \subfigure[]{\includegraphics[width=0.25\textwidth]{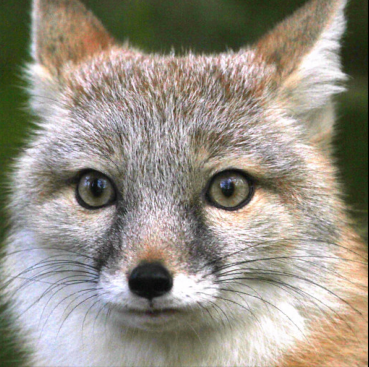}}
    \caption{Image regression task. (a) Ground truth image (source~\cite{FF_github}). (b) Image reconstructed using MLP. (c) Image reconstructed using the randoZeNN model.}
    \label{fox}
\end{figure}
%
%%%%%%%%%%%%%%%%%%%%%%%%%%%%%%%%%%%%%%%%%%%%%%%%%

To test our deep ZeNN models with two-dimensional data we consider the following image regression task: given an image, say Figure~\ref{fox}, we consider the function $F:{\mathcal D}\subset \bbR^2\rightarrow\bbR^3$ defined by mapping the (normalized) pixel coordinates $(x_1, x_2)\in\bbR^2$ to the (normalized) RGB codification $(y_1, y_2, y_3)\in\bbR^3$; the goal of the task is to approximate $F$ with a Neural Network. In a context of supervised learning we use $75\%$ of randomly selected pixel data as training data and use the remaining $25\%$ as a validation set. We track the success of the models in terms of the peak signal-to-noise ratio (PSNR) evaluated in the validation set; recall that in this case bigger is better. 

All the models we test have a similar global structure with 2d inputs, 3d outputs, with linear activation, and 4 hidden layers, with the last 3 being MLPs with ReLU activation and a fixed number of units. The distinction between the models occurs only at level of the first hidden layer: i) the {\em MLP} model has (as first hidden layer) another ReLU MLP layer; ii) the {\em oZeNN} model concatenates 2 oZeNN layers~\eqref{oZeNN}, one with sine activation and the other with cosine activation;  iii) the {\em radZeNN} model concatenates 2 radZeNN layers~\eqref{radZeNN}, one with sine activation and the other with cosine activation; iv) the {\em randoZeNN} model concatenates 2 randoZeNN layers~\eqref{randoZeNN}, one with sine activation and the other with cosine activation; v) the {\em FF} model uses a random Fourier features embedding~\eqref{FF}; vi) the {\em FF-trainable} model is similar to the previous but with matrix $B$ composed of trainable parameters. 

%An exhaustive search of hyperparameter space for these models is beyond the scope of this paper. 
The results we present correspond to the best performance we were able to achieve using models following the previous constraints, by searching (in a non-exhaustive way) through hyperparameter space. For instance, in each case, we increased the number of units in the first hidden layer until we found diminishing returns in terms of validation PSNR. The specifications concerning these models are given in Appendix~\ref{app:specs}. 

The first conclusion from our experiments (see Figure~\ref{PSNR_all} and Figure~\ref{fox}) is that, once again, ZeNNs excel at learning high frequencies;  the exception being radZeNNs, which nonetheless clearly outperform MLPs, that, as expected, have a very poor performance in this task. Moreover, oZeNNs and randoZeNNs have an excellent performance on par with the FF models. In fact, when we exclusively compare the 4 top models (see Figure~\ref{PSNR_top}), we observe that randoZeNN has a (slightly) better performance than the remaining 3 contenders, and that, somewhat surprisingly, oZeNN has a (slightly) weaker result, even though it has a considerably larger number of parameters. 

Finally, we stress that we introduced the FF-trainable models to make sure that the better performance of randoZeNN, when compared to that of FF, was not just a consequence of the fact that the first layer of  randoZeNN has trainable parameters, while the first layer of a vanilla FF model does not.     

%\section{Conclusion}

 %
\begin{figure}[t]
\centering
\includegraphics[width=7cm]{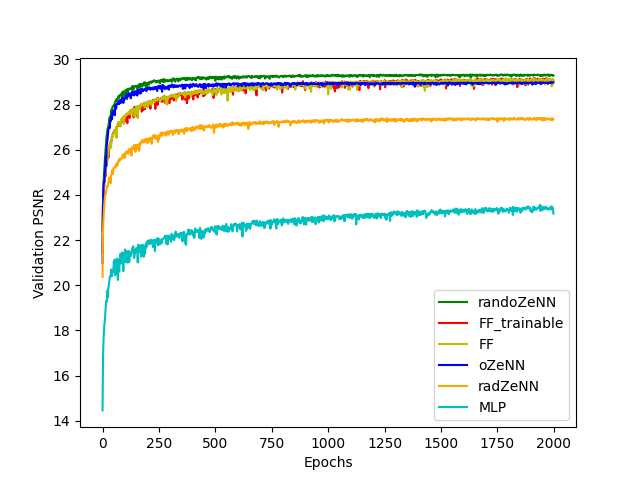}
\caption{ 
PSNR on validation set (bigger is better) along training in image regression task. Comparing 6 different architectures.   
}
\label{PSNR_all}
\end{figure}
\begin{figure}[t]
\centering
\includegraphics[width=7cm]{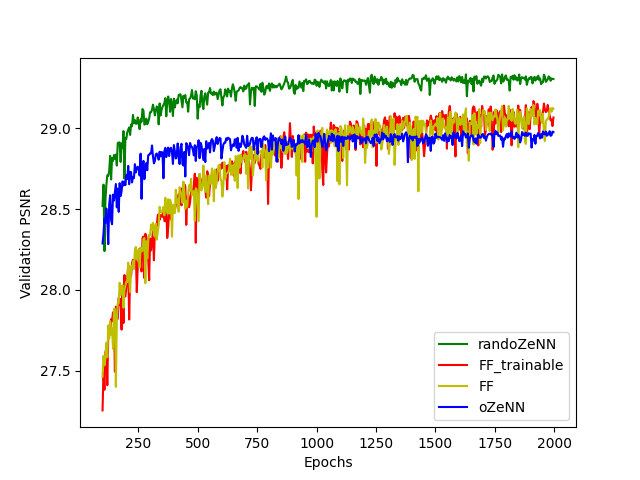}
\caption{ 
PSNR on validation set (bigger is better) along training in image regression task. Comparing the 4 top models.   
}
\label{PSNR_top}
\end{figure}

\subsection{PiNNs - solving differential equations using ZeNNs.}
\label{subsec:image}

We now report on the performance of (shallow) ZeNNs in solving differential equations. In coherence with our remaining experiments, we chose examples of equations whose solutions exhibit a highly oscillatory behavior. 

\subsubsection{Schr\"odinger's equation}

First we consider Schr\"odinger's equation, in $1+1$-dimensions and an infinite squared well potential: 
$$i \hbar \partial_t \psi(t,x) = - \frac{\hbar^2}{2m} \partial^2_x \psi(t,x) +V(x) \psi(t,x) $$
with the potential defined by
$$ V(x) = 
\left\{
\begin{array}{l}
0 \;\;,\;\text{ if } 0\leq x\leq 1 \;,\\
\infty \;\;,\;\text{ otherwise}\;.
\end{array}
\right.
$$
By separation of variables $\psi(t,x)=\tau(t)y(x)$, the spatial component of the solution is determined by the following boundary valued problem  
\begin{equation}
\label{IBVP1}
\left\{
\begin{array}{l}
y'' + k^2 y = 0 \\
y(0) = y(1) = 0\;,
\end{array}
\right.
\end{equation}
where $k = \frac{\sqrt{2mE}}{\hbar}$. The solutions of this problem are of the form 
\begin{equation} 
y(x) = A \sin(\pi n x)\quad \quad , \quad  k =n\pi \;,
\end{equation}
with $n$ a natural number that quantizes the energy level, and $A$ any real number; usually $A$ is fixed by imposing an extra normalization condition, which, in the context of quantum mechanics, is given by $\int_{-\infty}^{+\infty} y(x)dx=1$; here we will not impose any such condition, therefore we will allow the network to "choose" the constant $A$ freely.  

We consider this particular problem for the following reasons: 
\begin{itemize}
\item For ``large'' energies (large $n$'s) the solutions exhibit highly oscillatory behavior.  
\item As we saw, the exact solution is known, which provides a controlled setting to test our method. 
\item Finally, this is also an eigenvalue problem, meaning that besides finding the function $y=y(x)$ we also need to find the eigenvalue determined by $k$; this can be cumbersome in the context of classical methods, such as finite differences, but fits very naturally in the context of self-supervised learning with neural networks, i.e. in the context of PiNNs.
\end{itemize}

To solve initial boundary valued problem~\eqref{IBVP1}, we consider a parametric model 
\begin{equation}
\label{PiNN}
\begin{array}{rcl}
\Psi \, : \quad   \bbR^{4N+1} &\longrightarrow &  C(\R, \R)\times \bbR       \\
   (\theta, k) & \longmapsto & (y_{\theta},k)  
\end{array}
    \end{equation}
%%
%\begin{equation}
%\bbR^{4N+1}\ni (\theta, k) \mapsto (y_{\theta},k)\in\bbR^2\;,
%\end{equation}
%%
where $y_{\theta}:\bbR\rightarrow\bbR$ is a neural network model with parameter $\theta\in\bbR^{4N}$ and with architecture given by either~\eqref{e1.2} or~\eqref{e1.5}, and $k\in\bbR$ is an ``extra'' parameter designed to find the the solution's eigenvalue. Our goal is then to find parameters $(\theta^*,k^*)\in \bbR^{4N+1}$ such that the pair $(y_{\theta^*},k^*)$ solves~\eqref{IBVP1}. We search for approximate values for these optimal parameters by applying gradient descent 
\begin{equation}
\partial_t \theta_i (t) = - \partial_{\theta_i} {\mathcal L}(\theta,k)\quad\quad\text{ and } \quad\quad \partial_t k (t) = - \partial_{k} {\mathcal L}(\theta,k)\;,
\end{equation}
with loss given by 
\begin{equation}
{\mathcal L}(\theta,k) = \sum_{i=1}^M \left(y_{\theta}''(x_i)+k^2y_{\theta}(x_i)\right)^2 +
p_{\text{bdata}} \big(y_{\theta}^2(0)+y_{\theta}^2(1)\big)\;,
\end{equation}
where each $x_i\in [0,1]$, $i=1,\ldots,M$,  is a collection of training data points and $p_{\text{bdata}}\gg1$ is a parameter that tunes the importance of the boundary data in the loss function.   

We solved this problem using shallow networks (i.e. only 1 hidden layer, as in~\eqref{e1.2} and~\eqref{e1.5}) of 3 different kinds: i) {\em standard initialization MLPs}, corresponding to MLP models with relu or sine (inner) activation and standard weight initialization $W^{(1)},W^{(2)}\sim {\mathcal N}(0,1)$, regardless of the energy level $n$; ii)  {\em FF-trainable}, corresponding to MLPs with sine activations and adjustable inner weight initialization $W^{(1)}\sim {\mathcal N}(0,\rho^2)$ and $W^{(2)}\sim {\mathcal N}(0,1)$, which allows us to increase $\rho$ in order to find higher energy solutions; iii) {\em ZeNNs} with sine activation.   
Note that, in this shallow context, FF-trainable and SIREN models~\cite{sitzmann2019siren} are essentially the same.  

The results we obtained are easy to summarize: 

Standard initialization MLPs, with sine activation, are, in general, unable to find non trivial solutions of the problem at hand, if $n\geq 2$. With such architecture and initialization the neural network converges to the zero function and $k$ tends to diverge away from the desired value of $n\pi$; note that the zero function $y\equiv 0$, together with any choice of $k\in\bbR$, gives us a solution of the problem at hand, while non-trivial solutions are, modulo sign, related to a single choice of $k$. This tendency to converge to the trivial solution is a typical difficulty of using PiNNs to solve this type of problems and makes it even harder to find higher energy solutions. There are interesting proposals to circumvent these difficulties, for instance~\cite{Jin0} suggests adding a loss term that penalizes solutions with values close to zero and a loss term that penalizes the decrease of the eigenvalue below a specified threshold; using these and other techniques~\cite{Jin0} and~\cite{Jin} report solutions up to the 4th energy level; moreover their methods allow to use a previous solution as a stepping stone to finding the next solution in the energy ladder. 

In our experiments, we intentionally avoided imposing the referred extra penalty terms into our loss function to see if ZeNNs could determine higher energy solutions ``on their own''. It turns out that they can, see Figure~\ref{fig:Schrodinger10} for a remarkable fit of the 10th energy level solution, given a ZeNN with $N=64$ inner units and $\alpha=0.1$. Using a FF -learnable architecture with the same number of units and an appropriate choice of inner standard deviation ($\rho^2 = 30$) we obtained similar results. Nonetheless, we should warn that finding such high energy solutions required a good initial guess to the exact solution's eigenvalue $k^{*2}=(n\pi)^2$; in fact, we were required to choose $|k_0^2-k^{*2}|\approx1$ and after training we obtained $|k^2-k^{*2}|\approx 0.05$. This is clearly an undesirable featured, but should be compared with the fact that, even if we provide a standard initialization MLP with the correct eigenvalue at initialization, and use $N=512$ neurons, that architecture ends up converging, once again, to the trivial solution.           

\begin{figure}[t]
\centering
\includegraphics[width=7cm]{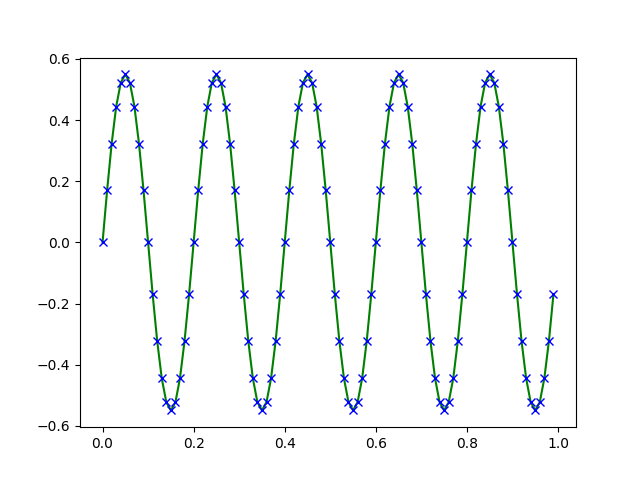}
\caption{ 
Solving Schr\"odinger's equations using ZeNNs, energy level $n=10$. The green line corresponds to the approximate solution generated by a shallow ZeNN, with 64 neurons and $\alpha=0.1$. The blue crosses are sample points form the exact solution with the same $L^2$ normalization.    
}
\label{fig:Schrodinger10}
\end{figure}

It should be possible to leverage the use of ZeNNs together with the methods in~\cite{Jin} in order to avoid our need to fine-tune the seed for the eigenvalue and, at the same time, allowing to surpass the energy bottlenecks implicit in~\cite{Jin}. 

\subsubsection{Sinc like solutions}

A natural criticism to the experiences in the previous section is that we obtained good results because we used linear comninations of sine functions, both with ZeNNs and Fourier Features, to approximate a solution that is exactly a sine wave. To show that ZeNNs generalize beyond such scope we consider the problem 
\begin{equation}
\label{IBVP2}
\left\{
\begin{array}{l}
xy'' + 2y' + L^2xy = 0 \\
y(0) = 0\;,
\end{array}
\right.
\end{equation}
where $L>0$ is a parameter. 

The solution to this problem is given by  
$$y(x) = \text{sinc}(Lx) = \frac{\sin(Lx)}{Lx}\;.$$
Once again, we find ourselves in a controlled setting and with means to leverage the frequency spectrum of our solutions -- the increase of the parameter $L$ leads to solutions with a larger frequency spectrum.   

We used the exact same 3 types of architectures discussed in the previous section. As expected, standard initialization MLPs provide good solutions, and fast convergence, if the oscillation parameter is small ($L\sim 1$), but are unable to find solutions once $L$ starts increasing. On the other hand ZeNNs excel at finding good solutions for large $L$ (see Figure~\eqref{fig:ZeNN_Sinc}), by using a number of inner units of the order of $L$, and FF - learnable networks also provide good solutions, in par with the ZeNN results, for instance by setting the number of units to be of order $L$ and the inner standard deviation $\rho^2$ to be of order $L/2$.  

%%
%\begin{figure}[t]
%\includegraphics[width=7cm]{Sinc_ZeNN_n100_alpha0.1_p_bdata_1e5__L100.png}
%\caption{ 
%Solving~\eqref{IBVP2} with $L=100$. The green line corresponds to the approximate solution generated by a shallow ZeNN, with 100 neurons and $\alpha=0.1$. The blue crosses are sample points form the exact solution.    
%}
%\centering
%\label{fig:ZeNN_Sinc}
%\end{figure}
%%

%
\begin{figure}[t]
\centering
\includegraphics[width=7cm]{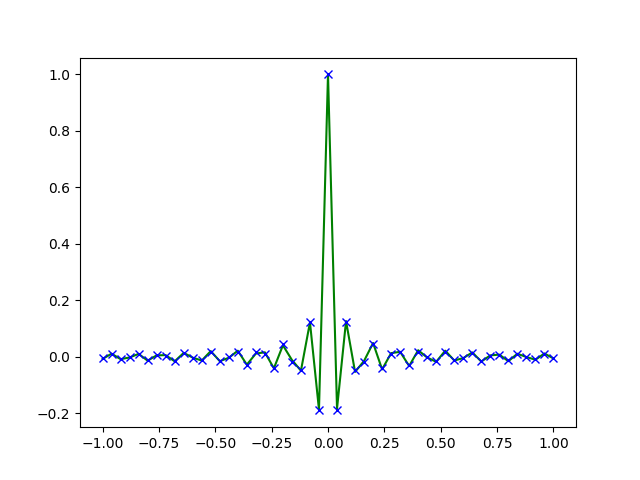}
\caption{ 
Solving~\eqref{IBVP2} with $L=100$. The green line corresponds to the approximate solution generated by a shallow ZeNN, with 100 neurons and $\alpha=0.1$. The blue crosses are sample points form the exact solution.    
}
\label{fig:ZeNN_Sinc}
\end{figure}

\subsection*{Acknowledgments}
The authors would like to thank M\'ario Figueiredo and Rui Rodrigues for interesting and stimulating conversations regarding the contents of this article.
This work was partially supported by FCT/Portugal through CAMGSD, IST-ID,
projects UIDB/04459/2020 and UIDP/04459/2020. 
J.L.C also acknowledges the support provided by FCT/Portugal and CERN through the project CERN/FIS-PAR/0023/2019. G. O. also acknowledges support by FCT 2021.02151.CEECIND.

\bibliographystyle{plainnat}
\bibliography{Refs}

\newpage
\appendix
\onecolumn

\section{Proof of Theorem \ref{thm:convergence}}
\label{s2}

We restate the theorem for convenience.
\begin{theorem}
    Let $\sigma$ be continuous and have at most polynomial growth of order $k \in \mathbb{N}$, i.e its satisfies $|\sigma(x)| \lesssim 1+|x|^k$. Further assume the measures $\widehat{\nu_j}$ to have uniformly bounded first $k$-moments. Then, for $\alpha>k+1$, the sequence of asymmetric maps $\widehat\Zc^N$ 
    \begin{equation}
\label{e1.12}
\begin{array}{rcl}
\widehat \Zc^N \, : \quad   \left(\R^4\right)^{\N} &\longrightarrow &  C(\R, \R)       \\
   \theta & \longmapsto & \Zc^N \circ P^{(N)} \, (\theta) =\,  f_{\theta}^N  
\end{array}
    \end{equation}
     converges uniformly in compact subsets of $\mathbb{R}$, and a.e. in $\left(\left(\R^4\right)^{\N}, \nu^\infty\right)$. In particular, for $\nu^\infty$-a.e. $\theta \in \left(\R^4\right)^{\N}$ and all $x \in \mathbb{R}$
    \begin{equation}
\widehat \Zc^\infty(\theta)(x)  = \lim_{N \to \infty} \widehat \Zc^N(\theta)(x) =
\sum_{j=1}^{\infty}  \frac{1}{j^\alpha}   \, \left(  W^{(2)}_{j} \ \sigma \left( j  \,  W^{(1)}_{j} x +  b^{(1)}_j \right) +  b^{(2)}_j \right) \, .
    \end{equation}
\end{theorem}

\begin{proof}

 Let $I$ be compact, and $x \in I$, furthermore, let $A:=\max \ \{|x|^k : x \in I\}$.  We know that,

\[
\left|  \sigma \left( j  \,  W^{(1)}_{j} x +  b^{(1)}_j \right)\right| \leq C \left( j^k \left| W^{(1)}_{j} \right|^k |x|^k +  |b^{(1)}_j |^k\right) \leq C \left( j^k \left| W^{(1)}_{j} \right|^k A +  |b^{(1)}_j |^k\right),
\]

with $C \geq 1$. Fix $\epsilon>0$ and use the Chebychev inequality to estimate

\begin{align*}
    \mathbb{P} \left[ \sup_{\substack{x \in I\\ N_1,N_2>N}}| \Fc^{N_2}_A(\theta)(x) - \Fc^{N_1}_A(\theta)(x) | > \epsilon \right] & \leq \frac{1}{\epsilon} \mathbb{E} \left[ \sup_{\substack{x \in I\\ N_1,N_2>N}} \left| \sum_{j=N_1}^{N_2}  \frac{1}{j^\alpha}   \, \left(  W^{(2)}_{j} \ \sigma \left( j  \,  W^{(1)}_{j} x +  b^{(1)}_j \right) +  b^{(2)}_j \right) \right| \right] \\
		& \leq \frac{1}{\epsilon} \mathbb{E} \left[ \sup_{\substack{x \in I\\ N_1,N_2>N}}  \sum_{j=N_1}^{N_2}  \frac{1}{j^\alpha}   \, \left| W^{(2)}_{j}\right| \ \left|\sigma \left( j  \,  W^{(1)}_{j} x +  b^{(1)}_j \right) \right|+ \left| b^{(2)}_j \right| \right] \\
& \leq \frac{1}{\epsilon} \mathbb{E} \left[ \sup_{x \in I}  \sum_{j=N}^{\infty}  \frac{1}{j^\alpha}   \, \left| W^{(2)}_{j}\right| \ \left|\sigma \left( j  \,  W^{(1)}_{j} x +  b^{(1)}_j \right) \right|+ \left| b^{(2)}_j \right| \right] \\
	& \leq \frac{C}{\epsilon} \mathbb{E} \left[ \sup_x  \sum_{j=N}^{\infty}  \frac{1}{j^\alpha}   \, \left| W^{(2)}_{j}\right| \left(j^k \left| W^{(1)}_{j} \right|^k A +  |b^{(1)}_j |^k\right) + \left| b^{(2)}_j \right| \right] \\
    & \leq \frac{C}{\epsilon} \ \sum_{j=N}^{\infty}  \frac{1}{j^\alpha}   \, \left(   \mathbb{E} [| W^{(2)}_{j}|] \  \mathbb{E} \left[ j^k \left| W^{(1)}_{j} \right|^k A +  |b^{(1)}_j |^k \right] +  \mathbb{E} [ |b^{(2)}_j |] \right)  \\
		%%%
		%%%
    & \leq \frac{C}{\epsilon} \ \sum_{j=N}^{\infty}  \frac{1}{j^\alpha}   \, \left(   \mathbb{E} [| W^{(2)}_{j}|] \   \left( j^k \mathbb{E} [ | W^{(1)}_{j} |^k ] A +  \mathbb{E}[ |b^{(1)}_j |^k ] \right) +  \mathbb{E} [ |b^{(2)}_j |] \right) \\
		%%%
		%%%
    & \leq \frac{C}{\epsilon} \ \sum_{j=N}^{\infty}  \left( \frac{A_k}{j^{\alpha-k}}  + \frac{B_k}{j^\alpha}  \right) ,
\end{align*} 
where
$$A_k = \mathbb{E} [| W^{(2)}_{j}|] \mathbb{E} [ | W^{(1)}_{j} |^k ]A , \ \ B_k =\mathbb{E} [| W^{(2)}_{j}|]  \mathbb{E}[ |b^{(1)}_j |^k ] + \mathbb{E} [ |b^{(2)}_j |], $$
are both independent of $j$. Hence, for any $N_2\geq N_1>N$, and under the assumption that $\alpha>k+1$ we find that 
\begin{align*}
    \mathbb{P} \left[ \sup_{\substack{x \in I\\ N_1,N_2>N}} | \Fc^{N_2}_A(\theta)(x) - \Fc^{N_1}_A(\theta)(x) | > \epsilon \right] & \leq \frac{C_k}{\epsilon N^{\delta}},
\end{align*} 
for some positive $\delta<\alpha-k-1$ and $C_k$.

 We can characterize the set of $\theta \in (\mathbb{R}^4)^{\mathbb{N}}$ for which $\lbrace \Fc^{N}_A(\theta) (x)\rbrace_{N \in \mathbb{N}}$ does not converge as
\begin{align*}
    \Omega & := \bigcup_{\epsilon>0} \Bigg\lbrace \theta \in (\mathbb{R}^4)^{\mathbb{N}} \ \Big| \ \forall_{N>0} \ \exists_{N_2,N_1>N}  , \ \text{such that} \ \sup_{x \in I} |\Fc^{N_2}_A(\theta)(x) - \Fc^{N_1}_A(\theta)(x)| > \epsilon \Bigg\rbrace\\
   %%%
    &= \bigcup_{\epsilon>0} \Bigg\lbrace \theta \in (\mathbb{R}^4)^{\mathbb{N}} \ \Big| \ \forall_{N>0} , \ \sup_{\substack{x \in I\\ N_1,N_2>N}} |\Fc^{N_2}_A(\theta)(x) - \Fc^{N_1}_A(\theta)(x)| > \epsilon \Bigg\rbrace\\
&= \bigcup_{m  \geq 0} \Bigg\lbrace \theta \in (\mathbb{R}^4)^{\mathbb{N}} \ \Big| \ \forall_{N>0} , \ \sup_{\substack{x \in I\\ N_1,N_2>N}} |\Fc^{N_2}_A(\theta)(x) - \Fc^{N_1}_A(\theta)(x)| > \epsilon_m \Bigg\rbrace,
\end{align*}
 where $\{\epsilon_m\}_{m \in \N}$ is a strictly decreasing sequence with $\epsilon_m \searrow 0$.

We can then say, with 
\[
A^{(m)}=\Bigg\lbrace \theta \in (\mathbb{R}^4)^{\mathbb{N}} \ \Big| \ \forall_{N>0} , \ \sup_{\substack{x \in I\\ N_1,N_2>N}} |\Fc^{N_2}_A(\theta)(x) - \Fc^{N_1}_A(\theta)(x)| > \epsilon_m \Bigg\rbrace,
\]
that 
\[ 
\Omega= \bigcup_{m \geq 0} A^{(m)}.
\]
At this point it is convenient to define
\[
A^{(m)}_N=\Bigg\lbrace \theta \in (\mathbb{R}^4)^{\mathbb{N}} \ \Big|  \sup_{\substack{x \in I\\ N_1,N_2>N}} |\Fc^{N_2}_A(\theta)(x) - \Fc^{N_1}_A(\theta)(x)| > \epsilon_m \Bigg\rbrace,
\]
then 
\[
A^{(m)}=\bigcap_{N\geq 1} A^{(m)}_N  \quad \text{ and } \quad \mathbb{P} \left( A^{(m)} \right) \leq \mathbb{P} \left( A^{(m)}_N \right),
\]
for any $N \in \N$. 
From our previous computation, we find that, for all $N \gg 1$ independently of $m$, we have
$$\mathbb{P} \left( A^{(m)}_N\right) \leq \frac{C_k}{\epsilon_m N^{\delta}}.$$
Hence, it follows from the fact that $\mathbb{P} \left( A^{(m)} \right) \leq \mathbb{P} \left( A^{(m)}_N \right)$ for all $N$ that
$$\mathbb{P} \left( A^{(m)} \right) \leq \lim_{N \to + \infty} \mathbb{P} \left( A^{(m)}_N \right) =0.$$
Therefore,
\[
\mathbb{P} \left(\Omega\right) \leq \sum_{m \geq 0} \mathbb{P} \left(A^{(m)}\right)=0. 
\]

\end{proof}

%%%%%%%%%%%%%%%%%%%%%%%%%%%%%%%%%%%%%%%%%%%%%%%%%%%%%%%%%%%%%%%% 
\section{Cumulants and characterisitcs of ZeNN's}\label{app:Cumulants}

The main goal of this Appendix is to establish the proof of Theorem \ref{thm:Cumulants}. We also take the opportunity to compute the characteristic function of a simple ZeNN that gives an explicit example of non-Gaussian behavior in the infinite width limit.

\subsection{Proof of Theorem \ref{thm:Cumulants}}
Recall that a ZeNN with $N \in \mathbb{N}$ neurons and one hidden layer determines the function
$$f_{\theta}^N(x) = \sum_{j=1}^N \frac{1}{j^\alpha} p_{\hat \theta_j}(jx)   ,$$
where $\hat \theta=(W^{(2)},b^{(2)}, W^{(1)},b^{(1)})$ and
$$p_{\hat \theta}(x)  = W^{(2)}\sigma(W^{(1)} x + b^{(1)}) + b^{(2)}, $$ 
is a standard perceptron. In Theorem \ref{thm:Cumulants} one considers all $\hat \theta_j \sim \hat \theta$ to be iid random variables, and for fixed $x \in \mathbb{R}$, regards both $p_{\hat \theta}(x)$ and $f_{\theta}^N(x)$ as a random variables whose $r$-cumulants are respectively denoted by $\lambda^{(r)}(x)$ and $\lambda^{(r)}_N(x)$. Then, Theorem \ref{thm:Cumulants} can be stated as follows.

\begin{theorem}
Let $r \in \mathbb{N}$. In the setup previously described
    $$\lambda^{(r)}_N(x) = \sum_{j=1}^N \frac{\lambda^{(r)}(jx)}{j^{r\alpha}}.$$
Furthermore, if $\sigma$ satisfies $|\sigma(x)| \lesssim 1+|x|^k$, and the moments of $\hat \theta$ are bounded. Then, for $\alpha>k+1$ the ZeNN converges pointwise a.e. and the limiting cumulants are finite and satisfy
    $$\lim_{N \to + \infty} \lambda^{(r)}_N(x) = \sum_{j=1}^{+\infty} \frac{\lambda^{(r)}(jx)}{j^{r\alpha}}.$$
\end{theorem}
\begin{proof}
The characteristic function of $f_{\theta}^N(x)$ is defined by
$$t \mapsto Z_N(x,t)= \mathbb{E}_{\theta_j} \left[ \exp \left( i t f_{\theta}^N(x) \right) \right],$$
and can be computed as follows
\begin{align*}
    Z_N(x,t) & = \mathbb{E}_{\theta} \left[ \exp \left( i t f_{\theta}^N(x) \right) \right] \\
    & = \mathbb{E}_{\theta} \left[ \exp \left( i \sum_{j=1}^N \frac{t}{j^\alpha} p_{\hat \theta_j} (jx) \right) \right] \\
    & = \prod_{j=1}^N \mathbb{E}_{\hat \theta} \left[ \exp \left( i \frac{t}{j^\alpha} p_{\hat \theta}(jx) \right) \right] \\
     & = \prod_{j=1}^N Z_1 \left( jx ,  \frac{t}{j^\alpha} \right).
\end{align*}
Then, the resulting cumulant generating function is given by
$$F_N(x,t) = \ln Z_N(x,t) = \sum_{j=1}^N F_1 \left( jx ,  \frac{t}{j^\alpha} \right) .$$
In particular, taking derivatives with respect to $t$, and using the chain rule, we find that
\begin{align*}
    \frac{d^r}{dt^r} F_N(x,t) & = \sum_{j=1}^N \frac{1}{j^{r\alpha}} \left( \frac{d^rF_1}{dt^r} \right) \left( jx ,  \frac{t}{j^\alpha} \right),
\end{align*}
and we can compute the $k$-th cumulant of $f_{\theta}^N(x)$ by evaluating this at $t=0$ yielding 
 $$\lambda^{(r)}_N(x) = \sum_{j=1}^N \frac{\lambda^{(r)}(jx)}{j^{r\alpha}},$$
which is the formula in the statement.

We must now turn to the problem of understanding the convergence of the cumulants $\lambda^{(r)}_N(x)$ when $N \to ´+ \infty$ which is equivalent to establishing the convergence of the series 
$$\sum_{j=1}^{+\infty} \frac{\lambda^{(r)}(jx)}{j^{r\alpha}}.$$
For this it is convenient to notice that
$$|\lambda^{(r)}(x)| \lesssim \sum_{\ell \geq 0} \sum_{r_1+ \ldots + r_\ell =r} \prod_{i=1}^\ell |m^{(r_i)}(x)|, $$
where 
$$m^{(r)}(x)= \mathbb{E}_{\hat \theta} \left[ p_{\hat \theta} (x)^r \right],$$
denotes the $r$-th moment of $p_{\hat \theta}(x)$. In particular, under the hypothesis that 
$$|\sigma(x)| \lesssim (1+|x|^k) \lesssim (1+|x|)^k$$ 
we find that 
\begin{align*}
    |m^{(r)}(x)| & \leq \mathbb{E}_{\hat \theta} \left[ | p_{\hat \theta} (x)|^r \right] \\
    & \lesssim  \mathbb{E}_{\hat \theta} \left[ 1+ | W^{(2)}|^r (1 + |W^{(1)} x + b^{(1)}|)|^{rk} + |b^{(2)} |^r \right] \\
    & \lesssim \mathbb{E}_{\hat \theta} \left[ 1+  | W^{(2)}|^r +  |W^{(2)}|^r |W^{(1)}|^{rk} |x|^{rk} +  |W^{(2)}|^r |b^{(1)}|^{rk} + |b^{(2)} |^r \right].
\end{align*}
Recall that we are also assuming that the moments of $\hat{\theta}$, $\mathbb{E}_{\hat \theta}[|\theta|^p] = \mu_p$, are bounded. Furthermore applying Hölder inequality we obtain the well known relation
\[
\mathbb{E}(|x|) = \mathbb{E}(|x| \cdot 1) \leq \mathbb{E}(|x|^p)^{1/p} \mathbb{E}(1^q)^{1/q} = \mathbb{E}(|x|^p)^{1/p},
\]
for $p>1$. Thus, with $q = k$, we get
\[
\mu_k = \mathbb{E}_{\hat \theta}[|\theta|^k] \leq \mathbb{E}_{\hat \theta}[|\theta|^{rk}]^{1/k} = \mu_{rk}^{1/k}.
\]
We know that 

\begin{align*}
|m^r(x)| \leq& \mathbb{E} \left( 1+ |W^{(2)}|^r + |W^{(2)}|^r |W^{(1)}|^{rk}   |x|^{rk} + |W^{(2)}|^r |b^{(1)}|^{rk}+ |b^{(2)}|^{r} \right)\\
\leq& 1 + \mu_r + \mu_r\mu_{rk}  |x|^{rk}+ \mu_r\mu_{rk}+\mu_r\\
\leq& 1 + 2 \mu_r + \mu_r\mu_{rk}(1 + |x|^{rk})\\
\leq& 1 + 2 \mu_{rk}^{1/k} + \mu_{rk}^{1/k+1} (1 + |x|^{rk})\\
\leq& (1 + |x|^{rk}) \left( 1 + 2 \mu_{rk}^{1/k} + \mu_{rk}^{1 + 1/k} \right)\\
\leq& (1 + |x|^{rk}) \left( 2 + 2 \mu_{rk}^{1/k} + 2 \mu_{rk}^{1 + 1/k} \right)\\
=& 2 \left( 1 + |x|^{rk} \right) \left( 1 + \mu_{rk}^{1/k} \left( 1 + \mu_{rk} \right) \right)\\
\leq& 2 \left( 1 + |x|^{rk} \right) \left( 1 + \left( 1 + \mu_{rk} \right)^{1 + 1/k} \right)\\
\leq&4 \left( 1 + |x|^{rk} \right)  \left( 1 + \mu_{rk} \right)^{1 + 1/k} \\
\leq&4 \left( 1 + |x| \right)^{rk}  \left( 1 + \mu_{rk} \right)^{1 + 1/k}.
\end{align*}

Inserting this into the previous bound for $\lambda^{(r)}(x)$ and defining
\[
 M_R = \max \left\{ \prod_{i=1}^\ell (1+\mu_{r_ik})^{1+1/k} : r_1+ \dots, r_\ell = r \text{ and } \ell \in \{1,\ldots, n\}  \right\},
\]
we obtain the following estimate
\begin{align*}
    \lambda^{r}(x) = &\sum_{\ell \geq 1} \sum_{r_1, \dots, r_\ell = r} \prod_{i=1}^\ell |m^{(r_i)}(x)| \\
    \leq &\sum_{\ell \geq 1} \sum_{r_1, \dots, r_\ell = r} \prod_{i=1}^\ell  (1+|x|)^{r_ik} (1+\mu_{r_ik})^{1+1/k} \\
	= &\sum_{\ell \geq 1} \sum_{r_1, \dots, r_\ell = r}   (1+|x|)^{rk} \prod_{i=1}^\ell (1+\mu_{r_ik})^{1+1/k} \\
    %%%
		\leq  &(1+|x|)^{rk}\sum_{\ell \geq 1} \sum_{r_1, \dots, r_\ell = r}    M_r \\
	=&C_r (1+|x|)^{rk}
\end{align*}

% Then, we obtain
% \begin{align*}
%     |m^{(r)}(x)| & \lesssim 1 + \mu_r +  \mu_r^{k+1}(1+ |x|^{rk} )  \lesssim   \mu_r^{k+1} \left( 1+|x| \right)^{rk}.
% \end{align*}
% Inserting this into the previous bound for $\lambda^{(r)}(x)$ gives the following estimate
% \begin{align*}
%     |\lambda^{(r)}(x)| & \lesssim \sum_{\ell \geq 0} \sum_{r_1+ \ldots + r_\ell =r} \prod_{i=1}^\ell  \mu_{r_i}^{k+1} \left( 1+|x| \right)^{r_ik} \\
%     & \lesssim r \cdot r! \cdot \left(\prod_{i=1}^\ell  \mu_{r_i}^{k+1} \right) \left( 1+|x| \right)^{rk} \\
%     & \lesssim C_r  \left( 1+|x| \right)^{rk}
% \end{align*}
Then, the terms inside the sum satisfies
\begin{align*}
    \Big\vert \frac{\lambda^{(r)}(jx)}{j^{r\alpha}} \Big\vert & \lesssim \frac{ \left( 1+|jx| \right)^{rk}}{j^{r\alpha}} .
\end{align*}
For $x=0$ and $r\alpha>1$ the series converges and so we now focus in the case when $x \neq 0$. Then, for large $j$ we may assume that $|jx|>1$ and therefore 
\begin{align*}
    \Big\vert \frac{\lambda^{(r)}(jx)}{j^{r\alpha}} \Big\vert  \lesssim  \frac{  |jx|^{rk} }{j^{r\alpha}}  =  \frac{|x|^{rk}}{j^{r(\alpha-k)}} ,
\end{align*}
and therefore the series converges for $r(\alpha-k)>1$ which is the case if $\alpha>k+1$.
\end{proof}

\begin{remark}
    The proof of the previous proposition contains two interesting formulas computing the moment and cumulant generating function of $f_{\theta}^N(x)$ in terms of those of $p_{\hat \theta}(x)$.
\end{remark}

\begin{remark}
    For $x=0$ we have
    $$\lambda^{(r)}_N(0) = \sum_{j=1}^N \frac{\lambda^{(r)}(0)}{j^{r\alpha}} = \lambda^{(r)}(0) \ \sum_{j=1}^N \frac{1}{j^{r\alpha}},$$
    and under the assumption that $r\alpha>1$ we can take the limit and obtain
    $$\lim_{N \to + \infty} \lambda^{(r)}_N(0) =  \lambda^{(r)}(0) \ \sum_{j=1}^{+\infty} \frac{1}{j^{r\alpha}}.$$
    In particular, it does not vanish if $\lambda^{(r)}(0) \neq 0$. We therefore find that $f_{\theta}^N(0)$ does not converge to a Gaussian distribution if any of the cumulants of order at least two of $p_{\hat \theta}(0)$ is nonzero.

    In fact, we further find that little information is lost when passing to the limit $N \to +\infty$. Indeed, suppose $\hat \alpha$ are $\hat \beta$ two sets of parameters drawn from two different distributions such that at least one of the cumulants of $p_{\hat \alpha}$ and $p_{\hat \beta}$ are different. Then, the $N \to +\infty$ limit of two ZeNN's with $\hat \theta_i \sim \hat \alpha $  and $\hat \theta_i \sim \hat \beta$ follow different distributions.
\end{remark}

\begin{remark}\label{rem:Cumulants of Standard NN}
    It is interesting to compare this result with the case for standard NN given by the formula $\frac{1}{\sqrt{N}} \sum_{j=1}^N p_{\hat \theta_j}(x)$. Then, we find that the corresponding characteristic function is
    \begin{align*}
        Z_N(x,t) & = \mathbb{E}_{\theta} \left[ \exp \left( i \frac{t}{\sqrt{N}} \sum_{j=1}^N p_{\hat \theta_j}(x) \right) \right] \\
        & = \prod_{j=1}^N \mathbb{E}_{\hat \theta} \left[ \exp \left( i \frac{t}{\sqrt{N}} \sum_{j=1}^N p_{\hat \theta}(x) \right) \right] \\
        & = \left( Z_1 \left( x ,  \frac{t}{\sqrt{N}} \right) \right)^N,
    \end{align*}
     and the cumulant generating function by
     \begin{align*}
         Z_N(x,t) & = N \ln Z_1 \left( x ,  \frac{t}{\sqrt{N}} \right) \\
         & = N F_1 \left( x ,  \frac{t}{\sqrt{N}} \right) \\
         & = N \left(  \sum_{r=1}^n \frac{\lambda_r(x)}{r!} \left( \frac{t}{N^{1/2}} \right)^r + o \left( \frac{t^n}{N^{n/2}} \right)\right).
     \end{align*}
     Under the assumption that $p_{\hat \theta}(x)$ has zero average, i.e. $\lambda_1(x)=0$, we find that
     \begin{align}
         Z_N(x,t) & =   \sum_{k=2}^n \frac{\lambda^{(k)}(x)}{k!}  \frac{t^k}{N^{k/2-1}}  + o \left( \frac{t^n}{N^{n/2}} \right),
     \end{align}
     and therefore
     $$\lambda^{(k)}_N = \frac{\lambda^{(k)}(x)}{N^{k/2-1}}. $$
     In particular, we see that for $k>2$ the limiting cumulant $\lim_{N \to + \infty}\lambda^{(k)}_N=0$ vanish as we were expecting because of the central limit theorem.
\end{remark}

\subsection{An example of a ZeNN characteristic function and non-Gaussian behavior}

Let $\sigma$ be the relu activation, and $W$ and $b$ be independent uniform random variables with densities $f_W = \frac{1}{2L}1_{[-L,L]}$ and $f_b = \frac{1}{2B}1_{[-B,B]}$. Let also $x$ be a positive real number and $j$ a positive integer. Then 
\begin{align*}
\mathbb{E}_{w\sim f_W, b\sim f_b}\left[ e^{it \sigma(jWx+b)}\right] 
&=
\int_{-\infty}^{+\infty}\int_{-\infty}^{+\infty} e^{it \sigma(jwx+b)} f_b(b)db f_w(w) dw
\\
&=
\frac{1}{2L} \int_{-L}^{L}
\left[
\int_{-jwx}^{+\infty} e^{it(jwx+b)} f_b(b) db +
\int_{-\infty}^{-jwx} f_b(b) db
\right]dw 
\\
&=
\frac{1}{4LB} \int_{-L}^{L}
\left[
\int_{\max\{-B,-jwx\}}^{\max\{B,-jwx\}} e^{it(jwx+b)} db +
\int_{\min\{-B,-jwx\}}^{\min\{B,-jwx\}}  db
\right]dw 
\\
&=: I + II \;.
\end{align*}

Set $m_1=\min\{-B,-jwx\}$ and $m_2=\min\{B,-jwx\}$, then 
\begin{align*}
II = \frac{1}{4LB} \int_{-L}^{L} \int_{m_1}^{m_2} db dw 
= 
\frac{1}{4LB} \int_{-L}^{L} (m_2-m_1)(w)dw\,. 
\end{align*}
By considering the change of variables $y=jwx$ we have
\begin{equation*}
(m_2-m_1) (y)
=
\left\{
\begin{array}{l}
2B \; , \; y<-B \\ 
-y+B \;, \; -B \leq y \leq B \\
0 \;, \; y>B \;
\end{array}
\right.
\end{equation*}
and, since $x\neq 0$, 
\begin{align*}
II = \frac{1}{4LB jx} \int_{-jxL}^{jxL} (m_2-m_1)(y)dy\,. 
\end{align*}
If $jLx\leq B$, we get 
\begin{align*}
II = \frac{1}{4LB jx} \int_{-jxL}^{jxL} -y+B\,dy =\frac{1}{2}\, 
\end{align*}
and if $jLx>B$ we also get 
\begin{equation*}
II = \frac{1}{4LB jx} \int_{-B}^{B} -y+B\,dy + \frac{|B-jxL|\,2B}{4LBjx}
= \frac{B+(jxL-B)}{2Ljx}=\frac{1}{2}\;. 
\end{equation*}

Now set $M_1=\max\{-B,-jwx\}$ and $M_2=\max\{B,-jwx\}$, so that 
\begin{align*}
I = \frac{1}{4LB} \int_{-L}^{L} \int_{M_1}^{M_2} e^{it(jwx+b)} db dw 
= 
\frac{1}{4LB i t} \int_{-L}^{L} e^{itjwx} \left(e^{itM_2}-e^{itM_1}\right) dw \;,  
\end{align*}
In the previous $y$ variable, the integrand function  of the last expression reads 
\begin{equation*}
g(y)
=
\left\{
\begin{array}{l}
2ie^{ity} \sin Bt  \; , \; y<-B \\ 
e^{it(y+B)}-1 \;, \; -B \leq y \leq B \\
0 \;, \; y>B \;, 
\end{array}
\right.
\end{equation*}
and
\begin{align*}
I = \frac{1}{4LB it jx} \int_{-jxL}^{jxL} g(y)dy\,. 
\end{align*}
If $jLx < B$ (recall $x>0$) we get 
\begin{align*}
I = \frac{1}{4LBitjx} \int_{-jxL}^{jxL} e^{it(y+B)}-1 dy
= 
\frac{1}{2B i t} \left( e^{itB}\frac{\sin(jxLt)}{jxLt}-1\right)\;,  
\end{align*}
and if $jLx \geq B$ we have 
\begin{align*}
I &= \frac{1}{4LBitjx} \int_{-B}^{B} e^{it(y+B)}-1 dy +  \frac{1}{4LBitjx} \int_{B}^{jxL} 2ie^{ity} \sin Bt dy
\\
&= 
\frac{1}{2i t j x L} \left( e^{itjxL}\frac{\sin(Bt)}{Bt}-1\right)\;.  
\end{align*}

Consider now that we have a sequences of independent random variables $w_j\sim f_w$ and $b_j\sim f_b$. Then  

\begin{align*}
\mathbb{E}_{w_j\sim f_W, b_j\sim f_b}\left[ e^{it \sum_{j=1}^N \sigma(jW_jx+b_j)/j^{\alpha}}\right] 
&=
\prod_{j=1}^N \mathbb{E}_{w\sim f_W, b\sim f_b}\left[ e^{it/j^{\alpha}  \sigma(jWx+b)}\right] 
\\
&=
\prod_{j=1}^{\lfloor{B/xL}\rfloor} 
\left[ \frac{1}{2B i t/j^{\alpha}} \left( e^{iBt/j^{\alpha}}\frac{\sin(xLt/j^{\alpha-1})}{xLt/j^{\alpha-1}}-1\right) +\frac12 \right]
\\
&\quad \times \prod_{j=\lceil{B/xL}\rceil}^N \left[\frac{1}{2ixL t/j^{\alpha-1} } \left( e^{ixLt/j^{\alpha-1}}\frac{\sin(Bt/j^{\alpha})}{Bt/j^{\alpha}}-1\right) +\frac12 \right]\;.
\end{align*}

Since the characteristic function of a Gaussian is a Gaussian, this provides an example of non-Gaussian behavior even in the infinite width limit, when $N\rightarrow+\infty$.

\section{Training and feature learning}
\label{app:Training}

%\subsection{Evolution of the loss during training}

Consider the loss function
$$\mathcal{L}^N(\theta) = \frac{1}{2} \sum_{\mu} \left( \Zc^N(\theta)(X_\mu) - Y_\mu   \right)^2 ,$$
and we consider $\lbrace \theta(t) \rbrace_{t \geq 0}$ learning by gradient descent, i.e.
$$\dot{\theta}(t)=-\nabla \mathcal{L}^N ( \theta(t) ) . $$
Then, we find that
\begin{align} \label{label_evolution}
    \partial_t \Zc^N(\theta(t))(x) = - \sum_{\mu} K_\theta^N (x,X_\mu) ( \Zc^N(\theta)(X_\mu) - Y_\mu),
\end{align}
where 
\begin{align}\label{eq:NTK}
K_\theta^N (x,y) & = xy \sum_{j=1}^N \frac{\left( W^{(2)}_j \right)^2}{j^{2\alpha-2}} \dot{\sigma} (j W_1^{(j)}x+b^{(1)}_j) \dot{\sigma} (j W_1^{(j)}y+b^{(1)}_j) \\ \nonumber
& \ \ \ \ + \sum_{j=1}^N \frac{1}{j^{2\alpha}} \left[ 1 + \sigma (j W_1^{(j)}x+b^{(1)}_j) \sigma (j W_1^{(j)}y+b^{(1)}_j) + \left( W^{(2)}_j \right)^2 \dot{\sigma} (j W_1^{(j)}x+b^{(1)}_j) \dot{\sigma} (j W_1^{(j)}y+b^{(1)}_j)  \right] 
\end{align}
will be referred to as the Zeta neural tangent kernel (ZeNTK).

From \eqref{label_evolution} with \( E_\mu(t) = \Zc^N(\theta(t))(x_\mu) - y_\mu \) and \( E^\intercal = (E_\mu)_{\mu \in [n]} \), it can be deduced that

$$
\dot{E} = K_\theta^N E
$$
and $ E^{\top} {\dot{E}} = E^{\top} K_\theta^N E $, 

that is \( \dot{L} = E^{\top} K_\theta^N E \), (since \( L = \frac{1}{2} E^{\top} E \) and hence 
\(\dot{L} = E^{\top} \dot{E} \)). We know that
\[
 K_\theta^N = \bigg[K_\theta^N(x_\mu, x_\nu)\bigg]_{\mu, \nu \in [n]}
 \]
 is symmetric so \( K_\theta^N = U^{\top} \Lambda U \) where \( \Lambda \) is diagonal and \( U \) orthogonal matrices.

$$
\dot{L} 
= - E^{\top} 
U^{\top}
 \Lambda U E \leq -\lambda_1(t) E^{\top} U^{\top} U E = -\lambda_1 E^{\top} E = -2\lambda_1(t) L
$$
where \( \lambda_1(t) \) is the smallest eigenvalue of \( \Theta_{\theta_(t)}^N \) and from Grönwall's inequality we conclude
\[
L(t) \leq e^{-2\int_0^t \lambda_1(\tau) d\tau} L(0).
\]
We know that the ZeNTK \( K_{\theta_(\tau)}^N \)  is positive semi-definite and thus \( \lambda_1(\tau) \geq 0 \), for any \( \tau \).

An argument similar to that of the proof of Theorem \ref{thm:convergence} shows the following.

\begin{proposition}
    Let $\sigma$ and $\dot{\sigma}$ be continuous with at most polynomial growth of order $k$. If the measures $\widehat{\nu_j}$ have uniformly bounded first moments and $\alpha>k+3/2$, then the sequence 
    \begin{equation}
\label{e1.12}
\begin{array}{rcl}
K^N \, : \quad   \left(\R^4\right)^{\N} &\longrightarrow &  C(\R \times \R, \mathbb{R} )       \\
   \theta & \longmapsto & K_\theta^N 
\end{array}
    \end{equation}
     converges uniformly on compact subsets of $\R \times \R
      $ and a.e. in $\left(\left(\R^4\right)^{\N}, \nu^\infty\right)$, i.e.
     \begin{equation}
     \label{e63}
     \exists \lim_{N \to \infty} \, K_\theta^N  =:
     K_\theta^\infty \, . 
     \end{equation}
     
     \end{proposition}

As we shall now see, unlike the MLP case, the limit in (\ref{e63}), $K_\theta^\infty$, depends crucially on the parameters $\theta$ and therefore allows for direct feature learning. 

\begin{theorem}[ZeNNs learn features, even in the infinetly wide limit]\label{thm:Features Appendix}
Let $x, y \in \mathbb{R}$ and suppose $\dot{\sigma} \neq 0$ in a non-empty interval. Then, there is an open set $U \subset \mathbb{R}^{4\mathbb{N}}$ such that we have
    $$\theta \in U \mapsto K_{\theta}^N (x,y) ,$$
is not constant. We also have that for all $\theta \in U$ such that the limit $\lim_{N \to +\infty}K_{\theta}^N (x,y) = K_{\theta}^\infty(x,y)$ exists, $\theta \mapsto K_{\theta}^\infty(x,y) $ is not constant.

Furthermore, if $\dot{\sigma} \neq 0$ except at a finite set of points, then the set $U$ can be taken to have full measure.
\end{theorem}
\begin{proof}
We can write the ZeNTK more concisely as
\begin{align*}
K_\theta^N (x,y) & = \sum_{j=1}^N \frac{1}{j^{2\alpha}} + K_{(1)}^N (x,y) + K_{(2)}^N (x,y),
\end{align*}
where
\begin{align*}
K_{(1)}^N (x,y) & = \sum_{j=1}^N \frac{1}{j^{2\alpha}} \sigma (j W_1^{(j)}x+b^{(1)}_j) \ \sigma (j W_1^{(j)}y+b^{(1)}_j) , \\
K_{(2)}^N (x,y) & =  \sum_{j=1}^N \frac{1}{j^{2\alpha-2}} \left( xy + \frac{1}{j^2} \right) \left( W^{(2)}_j \right)\dot{\sigma} (j W_1^{(j)}x+b^{(1)}_j)  \left( W^{(2)}_j \right) \ \dot{\sigma} (j W_1^{(j)}y+b^{(1)}_j).
\end{align*}
We can now check that $\theta \mapsto K_\theta^N(x,y)$ is in general not constant, not even in the limit $N \to + \infty$. This means that the kernel $K_\theta$ depends on the parameters $\theta$. To do this we simply take two derivatives with respect to $W^{(2)}_j$ for some fixed $j \in \mathbb{N}$, yielding
\begin{align*}
\frac{\partial^2 K_{(1)}^N (x,y)}{\left( \partial W^{(2)}_j \right)^2} & = 0 , \\
\frac{\partial^2 K_{(2)}^N (x,y)}{\left( \partial W^{(2)}_j \right)^2} & =  \frac{1}{j^{2\alpha-2}} \left( xy + \frac{1}{j^2} \right) \dot{\sigma} (j W_1^{(j)}x+b^{(1)}_j) \ \dot{\sigma} (j W_1^{(j)}y+b^{(1)}_j).
\end{align*}
Hence,
$$\frac{\partial^2 K_{\theta}^N (x,y)}{\left( \partial W^{(2)}_j \right)^2}  =  \frac{1}{j^{2\alpha-2}} \left( xy + \frac{1}{j^2} \right) \dot{\sigma} (j W_1^{(j)}x+b^{(1)}_j) \ \dot{\sigma} (j W_1^{(j)}y+b^{(1)}_j),$$
which is independent of $N$ and does not vanish for almost every $W^{(1)}_j$, $b^{(1)}_j$. This proves our claim that $K_{\theta}^N (x,y)$ is not constant for any finite $N$ nor can it become constant in the limit $N \to +\infty$.
\end{proof}

\section{Model specifications used in Section~\ref{subsec:image}}
\label{app:specs}

Here we present the detailed specifications of the models discussed in Section~\ref{subsec:image}. 
Recall that all such models have a similar global structure with 2d inputs, 3d outputs with linear activation, and 4 hidden layers. The last 3 hidden layers are MLPs with ReLU activation and $256$ units. The distinction between the models occurs only at level of the first hidden layer. Unless indicated, all weights were initialized with a unit normal distribution and all bias initialized at zero. 

We now provide the specifications for the first hidden layers:  
\begin{itemize}
\item {\em MLP} uses a ReLU MLP layer with $2048$ units; 
\item {\em oZeNN} concatenates 2 oZeNN layers~\eqref{oZeNN}, one with sine activation and the other with cosine activation,  both with $N=256$ and $\alpha=0$;  
\item {\em radZeNN} concatenates 2 radZeNN layers~\eqref{radZeNN}, one with sine activation and the other with cosine activation,   both with $N=256$ and $\alpha=0$; 
\item  {\em randoZeNN}  concatenates 2 randoZeNN layers~\eqref{randoZeNN}, one with sine activation and the other with cosine activation, both with $M=16384$, $N=256$ and $\alpha=0$; 
\item  {\em FF}  uses a random Fourier features embedding~\eqref{FF}, with $N=4096$ and  $\rho=10$; 
\item  the {\em FF-trainable} uses a trainable random Fourier features embedding, with $N=4096$ units and $\rho=10$. 
\end{itemize}
%%%%%%%%%%%%%%%%%%%%%%%%%%%%%%%%%%%%%%%%%%%%%%%%%%%%%%%%%%%%%%%% 

\end{document}